\documentclass[runningheads]{llncs}

\usepackage{hyperref}
\usepackage{placeins}
\usepackage{amssymb}
\usepackage{amstext}
\usepackage{todonotes}
\usepackage{amsthm}
\usepackage{amsmath}
\usepackage{array}
\usepackage{dsfont}
\allowdisplaybreaks

\newtheorem{thm}{\protect\theoremname}
\theoremstyle{plain}
\newtheorem{lem}[thm]{\protect\lemmaname}
\newtheorem{cor}[thm]{\protect\corname}
\makeatother
\usepackage[english]{babel} 
\providecommand{\lemmaname}{Lemma}
\providecommand{\theoremname}{Theorem}
\providecommand{\corname}{Corollary}

\begin{document}
\sloppy
\newcommand{\olivier}[1]{\textcolor{black}{#1}}
\newcommand{\otc}[1]{\textcolor{black}{#1}}
\newcommand{\crc}[1]{}%\todo[color=magenta!40]{Cl\'ement: #1}}

\newcommand{\carolanote}[1]{}%\todo[color=blue!40]{C: #1}}
\newcommand{\carolanotetwo}[1]{}%\todo[color=yellow!40]{C: #1}}
\newcommand{\laurentnote}[1]{}%\todo[color=red!40]{L: #1}}
\newcommand{\oliviernote}[1]{}%\todo[color=red!40]{O: #1}}
\title{On averaging the best samples in evolutionary computation} %A single parent is (provably) not enough}
%\title{Proving that  }
%\title{Choosing the right selection rate}
%\title{Choosing the right selection rate in the initial epoch}
\renewcommand{\topfraction}{.99}
\author{Laurent Meunier\inst{1}\inst{2}\inst{3}, Yann Chevaleyre\inst{2}, Jeremy Rapin\inst{1},\\ Cl\'ement W. Royer\inst{2},  Olivier Teytaud\inst{1}}
\authorrunning{L. Meunier, Y. Chevaleyre, J. Rapin, C. W. Royer, O. Teytaud}
\institute{Facebook Artificial Intelligence Research (FAIR), Paris, France\\
\and 
LAMSADE, CNRS, Universit\'e Paris-Dauphine, Universit\'e PSL, Paris, France
\and
Corresponding author: \textbf{\href{mailto:laurentmeunier@fb.com}{laurentmeunier@fb.com}}}
%LMeunier
%CRoyer - LAMSADE, CNRS, Universit\'e Paris-Dauphine, Universit\'e PSL, 75016 PARIS, France. clement.royer@dauphine.psl.eu
%OTeytaud
%YannC
%JR
\def\R{\mathbb{R}}
\maketitle
\begin{abstract}
Choosing the right selection rate is a long standing issue in evolutionary computation. \otc{In the continuous unconstrained case, }we prove mathematically that a single parent $\mu=1$ leads to a sub-optimal simple regret in the case of the sphere function. We provide a theoretically-based selection rate $\mu/\lambda$ that leads  to better progress rates. With our choice of selection rate, we get a provable regret of order $O(\lambda^{-1})$ which has to be compared with $O(\lambda^{-2/d})$ in the case where $\mu=1$. We complete our study with experiments to confirm our theoretical claims. %This lead to a selection rate:  %We solve this in the single-epoch case for the continuous case and the sphere function: we prove an explicit formula for the selection rate $\mu/\lambda$, depending on the population size $\lambda$ and the dimension $d$: $\mu=\lceil \lambda/1.1^d\rceil$.
\end{abstract}
%\setcounter{tocdepth}{5}
%\tableofcontents
\section{Introduction}
In evolutionary computation, the selected population size often depends linearly on the total population size, with a ratio between $1/4$ and $1/2$: $0.270$ is proposed in~\cite{escompr}, ~\cite{HAN,cmsa} suggest $1/4$ and $1/2$.
However,
some sources~\cite{amorales} recommend a lower value $1/7$. Experimental results in~\cite{ratio} and theory in \cite{fournierAlgorithmica} together suggest a ratio $\min(d,\lambda/4)$ with $d$ the dimension, i.e. keep a population size at most the dimension. \cite{jeb} suggests to keep increasing $\mu$ besides that limit, but slowly enough so that that rule $\mu=\min(d,\lambda/4)$ would be still nearly optimal. Weighted recombination is common~\cite{weightdirk}, but not with a clear gap when compared to truncation ratios~\cite{esniko}, except in the case of large population size~\cite{sumotori}.
There is, overall, limited theory around the optimal choice of $\mu$ for optimization in the continuous setting. 
In the present paper, we focus on a simple case (sphere function and single epoch), but prove exact theorems. 
We point out that the single epoch case is important \otc{by} itself - this is fully parallel optimization~\cite{nie,mckay,bergstra2012random,bousquet}.
Experimental results with a publicly available platform support the approach.
\section{Theory}
We consider the case of a single batch of evaluated points. We generate $\lambda$ points according to some probability distribution. We then select the $\mu$ best and average them. The result is our approximation of the optimum. This is therefore an extreme case of evolutionary algorithm, with a single population; this is commonly used for e.g. hyperparameter search in machine learning~\cite{bergstra,bousquet}, though in most cases with the simplest case $\mu=1$.
\subsection{Outline}
We consider the optimization of the simple function $x\mapsto ||x-y||^2$ for an unknown $y\in \mathcal{B}(0,r)$.
In Section \ref{notations} we introduce notations.
In Section \ref{easier} we analyze the case of random search uniformly in a ball of radius $h$ centered on $y$. \otc{We can, therefore, exploit the knowledge of the optimum's position and assume that $y=0$.} We then extend the results to random search in a ball of radius $r$ centered on $0$, provided that $r>||y||$ and show that results are essentially the same up to an exponentially decreasing term (Section \ref{harder}).
\subsection{Notations}\label{notations}
We are interested in minimizing the function $f:x\in\mathbb{R}^d\mapsto ||x-y||^2$ for a fixed unknown $y$ in parallel one-shot black box optimization, i.e. we sample $\lambda$ points $X_1,...,X_\lambda$ from some distribution $\mathcal{D}$ and we search for $x^\star = \arg\min_x f(x)$. In what follows we will study the sampling from $\mathcal{B}(0,r)$, the uniform distribution on the $\ell_2$-ball of radius $r$; w.l.o.g. $\mathcal{B}(y,r)$ will also denote the $\ell_2$-ball centered in $y$ and of radius $r$.  \\
We are interested in comparing the strategy ``$\mu$-best'' vs ``$1$-best''. We denote $X_{(1)},...,X_{(\lambda)}$, the sorted values of $X_i$ \otc{i.e. $(1)$,\dots,$(\lambda)$ are such that} $f(X_{(1)})\leq...\leq f(X_{(\lambda)})$. The ``$\mu$-best'' strategy is to return  $\bar{X}_{(\mu)} = \frac1\mu\sum_{i=1}^\mu X_{(i)}$ as an estimate of the optimum and the ``1-best'' is to return $X_{(1)}$. We will hence compare :
$\mathbb{E}\left[f\left(\bar{X}_{(\mu)}\right)\right]$ and $\mathbb{E}\left[f\left(X_{(1)}\right)\right]$. 
We recall the definition of the gamma function $\Gamma$: $\forall z>0$, $\Gamma(z) = \int_0^{\infty}t^{z-1}e^{-t}dt$, as well as the property $\Gamma(z+1)=z\Gamma(z)$.
% \begin{itemize}
% \item $\mathcal{D}$ distribution on $\mathcal{X}\subseteq\mathbb{R}^{d}$
% \item $S=(X_{1}\ldots X_{\lambda})$ where $X_{i}$ are random vectors drawn from
% $\mathcal{D}$
% \item $X^{j}$ is the $j^{th}$ component of vector $X$.
% \item $X_{(1)}\ldots X_{(N)}=sort\left(X_{1}\ldots X_{n}\right)$. So $X_{(1)}=\min_{i\in[N]}X_{i}$
% \item Let $\bar{X}_{(\mu)}=\frac{1}{\mu}\sum_{i=1}^{\mu}X_{(\mu)}$
% \item $f:\mathcal{X}\rightarrow\mathbb{R}$. For now, to simplify, let $f(x)=\left\Vert x-y\right\Vert ^{2}$
% where $y\in\mathcal{X}$ is fixed.
% \item Let $\mathcal{X}_{h}=\left\{ x\in\mathcal{X}:f(x)\le h\right\} $
% for any $h\in\mathbb{R}$
% \item Let $\mathcal{D}_{h}$ be the distribution defined as $\mathbb{P}_{X\sim\mathcal{D}_{h}}\left[\cdot\right]=\mathbb{P}_{X\sim\mathcal{D}}\left[\cdot\mid X\in\mathcal{X}_{h}\right]$.
% \item $\mathcal{B}\left(z,r\right)$ refers to the uniform distribution
% over the $\ell_{2}$ ball centered on $z$ and of radius $r$.
% \item The expected value of the ``k-best averages'' is $\mathbb{E}_{S}\left[f\left(\bar{X}_{(\mu)}\right)\right]$
% and the expected value of the ``1-best'' is $\mathbb{E}_{S}\left[f\left(X_{(1)}\right)\right]$.
% \item $\Gamma$ is the gamma function : $\forall z>0$, $\Gamma(z) = \int_0^{\infty}t^{z-1}e^{-t}dt$. We also remind the property $\Gamma(z+1)=z\Gamma(z)$
% \end{itemize}
\subsection{When the center of the distribution is also the optimum} \label{easier}
In this section we assume that $y=0$ (i.e. $f(x)=||x||^2$) and consider sampling in $\mathcal{B}(0,r)\subset \mathbb{R}^d$. In this simple case, we show that keeping the best $\mu>1$ sampled points is asymptotically a better strategy than selecting a single best point. The choice of $\mu$ will be discussed in Section~\ref{harder}.
%\otc{In particular it is, asymptotically, a better strategy than taking only one. We will see later how to non-asymptotically choose $\mu$. }\laurentnote{This part is not asymptotic, this is the next one}
\begin{thm}
\label{thm:k_avg_best_vs_1_best}For all $\lambda>\mu\geq2$  and $d\geq2,$ $r>0$,
for $f(x)=\left\Vert x\right\Vert ^{2}$,
\[
\mathbb{E}_{X_1,...,X_\lambda\sim \mathcal{B}(0,r)}\left[f\left(\bar{X}_{(\mu)}\right)\right]<\mathbb{E}_{X_1,...,X_\lambda\sim \mathcal{B}(0,r)}\left[f\left(X_{(1)}\right)\right].
\]
\end{thm}
To prove this result, we will compute the value of $\mathbb{E}\left[f\left(\bar{X}_{(\mu)}\right)\right]$ for all $\lambda$ and $\mu$. The following lemma gives a simple way of computing the expectation of a function depending only on the norm of its argument.
\begin{lem}
\label{lem:ball_volume} Let $d\in \mathbb{N}^*$. Let $X$ be drawn uniformly in $\mathcal{B}(0,r)$ the $d$-dimensional ball of radius $r$.
Then for any measurable function $g:\mathbb{R}\rightarrow\mathbb{R}$, we have 
\begin{align*}
\mathbb{E}_{X\sim\mathcal{B}(0,r)}\left[g\left(\left\Vert X\right\Vert \right)\right] & =\frac{d}{r^d}\int_{0}^{r}g\left(\alpha\right)\alpha^{d-1}d\alpha.
\end{align*}
In particular, we have
%\begin{align*}
$\mathbb{E}_{X\sim\mathcal{B}(0,r)}\left[\left\Vert X\right\Vert ^{2}\right] =\frac{d}{d+2}\times r^{2}$.
%\end{align*}
\end{lem}
\begin{proof}
Let $V(r,d)$ be the volume of a ball of radius $r$ in $\mathbb{R}^d$ and $S(r,d)$
be the surface of a sphere of radius $r$ in $\mathbb{R}^d$. Then
 $\forall r>0, V(r,d)=\frac{\pi^{d/2}}{\Gamma\left(\frac{d}{2}+1\right)}r^{d}$ and
 $S(r,d-1)=\frac{2\pi^{d/2}}{\Gamma\left(\frac{d}{2}\right)}r^{d-1}.$
Let $g:\mathbb{R}\rightarrow\mathbb{R}$ be a continuous function. Then:
\begin{align*}
\mathbb{E}_{X\sim\mathcal{B}(0,r)}\left[g\left(\left\Vert X\right\Vert \right)\right] & =
\frac{1}{V(r,d)}\int_{x:||x||\leq r} g(||x||)dx\\
&=\frac{1}{V(r,d)}\int_{\alpha=0}^{r}\int_{\theta:||\theta|| = \alpha} g(\alpha)d\theta d\alpha\\
&=\frac{1}{V(r,d)}\int_{\alpha=0}^{r} g(\alpha)S(\alpha,d-1) d\alpha\\
&=\frac{S(1,d-1)}{V(r,d)}\int_{\alpha=0}^{r} g(\alpha)\alpha^{d-1}d\alpha=\frac{d}{r^d}\int_{\alpha=0}^{r} g(\alpha)\alpha^{d-1}d\alpha.\\
%\end{align*}
%This yields
%\begin{align*}
\mbox{So,~}\mathbb{E}_{X\sim\mathcal{B}(r)}\left[\left\Vert X\right\Vert ^{2}\right] & = \frac{d}{r^d}\int_{\alpha=0}^{r} \alpha^2\alpha^{d-1}d\alpha\\
&=\frac{d}{r^d}\left[\frac{\alpha^{d+2}}{d+2}\right]^r_0=\frac{d}{d+2}r^2.
\end{align*}
\end{proof}
We now use the previous lemma to compute the expected regret\cite{bubeck2009pure} of the average of the $\mu$ best points conditionally to the value of $f(X_{(\mu+1)})$. The trick of the proof is that, conditionally to  $f(X_{(\mu+1)})$, the order of $X_{(1)},...,X_{(\mu)}$ has no influence over the average. Computing the expected regret conditionally to $f(X_{(\mu+1)})$ thus becomes straightforward.
\begin{lem}\label{lm3}
\label{lem:expectation_k_avg}For all $d>0$, $r^2>h>0$ and $\lambda>\mu\geq 1$, for $f(x)=\left\Vert x\right\Vert ^{2}$, %the expectation of
\[
\mathbb{E}_{X_1,...,X_\lambda\sim \mathcal{B}(y,r)}\left[f\left(\bar{X}_{(\mu)}\right)\mid f(X_{(\mu+1)})=h\right]=\frac{h}{\mu}\times\frac{d}{d+2}.
\]
\end{lem}
\begin{proof}
Let us first compute $\mathbb{E}\left[f\left(\bar{X}_{(\mu)}\right)\mid f(X_{(\mu+1)})=h\right]$.
Note that for any function $g:\mathbb{R}^d\rightarrow\mathbb{R}$ and distribution $\mathcal{D}$, we
have 
\begin{align*}
\mathbb{E}_{X_{1}\ldots X_{\lambda}\sim\mathcal{D}}&\left[g(\bar{X}_{(\mu)})\mid f(X_{(\mu+1)})=h\right]\\ &=\mathbb{E}_{X_{1}\ldots X_{\mu}\sim\mathcal{D}}\left[g\left(\frac{1}{\mu}\sum_{i=1}^{\mu}X_{i}\right)\mid X_{1}\ldots X_{\mu}\in\{x:f(x)\leq h\}\right]\\
 & =\mathbb{E}_{X_{1}\ldots X_{\mu}\sim\mathcal{D}_{h}}\left[g\left(\frac{1}{\mu}\sum_{i=1}^{\mu}X_{i}\right)\right],
\end{align*}
where $\mathcal{D}_h$ is the restriction of $\mathcal{D}$ to the level set $\{x:f(x) \le h\}$. In our setting, we have $\mathcal{D}=\mathcal{B}(0,r)$ and $\mathcal{D}_h = \mathcal{B}(0,\sqrt{h})$. Therefore,
\begin{align*}
 \mathbb{E}_{X_1,...,X_\lambda\sim\mathcal{B}(0,r)}&\left[f\left(\bar{X}_{(\mu)}\right)\mid f(X_{(\mu+1)})=h\right]\\
 &=\mathbb{E}_{X_1,...,X_\lambda\sim\mathcal{B}(0,r)}\left[||\bar{X}_{(\mu)}||^2\mid f(X_{(\mu+1)})=h\right]  \\
 & =\mathbb{E}_{X_{1}\ldots X_{\mu}\sim\mathcal{B}(0,\sqrt{h})}\left[||\frac1\mu\sum_{i=1}^\mu X_i||^2\right]\\
 & =\frac{1}{\mu^2}\mathbb{E}_{X_{1}\ldots X_{\mu}\sim\mathcal{B}(0,\sqrt{h})}\left[\sum_{i,j=1}^\mu X_i^TX_j\right]\\
 & = \frac{1}{\mu^2}\sum_{i,j=1,i\neq j}^\mu \mathbb{E}_{X_{i}\ldots X_{j}\sim\mathcal{B}(0,\sqrt{h})}\left[X_i^TX_j\right]\\
 &+\frac{1}{\mu^2}\sum_{i=1}^\mu\mathbb{E}_{X_{i}\sim\mathcal{B}(0,\sqrt{h})}\left[||X_i||^2\right]
 = \frac1\mu \mathbb{E}_{X\sim\mathcal{B}(0,\sqrt{h})}\left[||X||^2\right].
%  &+\left.\left(\mathbb{E}_{X_{1}\ldots X_{\mu}\sim\mathcal{D}_{h}}\left[\frac{1}{\mu}\sum_{i=1}^{\mu}X_{i}^{j}-y^{j}\right]\right)^{2}\right\} \\
%  & =\sum_{j=1}^{d}\left\{ \frac{1}{k^{2}}\sum_{i=1}^{\mu}Var_{\mathcal{D}_{h}}\left(X_{i}^{j}\right)+\left(\mathbb{E}_{\mathcal{D}_{h}}\left[\frac{1}{\mu}\sum_{i=1}^{\mu}X_{i}^{j}-y^{j}\right]\right)^{2}\right\} \\
%  & =\sum_{j=1}^{d}\left\{ \frac{1}{\mu}Var_{X\sim\mathcal{D}_{h}}\left(X^{j}\right)+\left(\mathbb{E}_{X\sim\mathcal{D}_{h}}\left[X^{j}-y^{j}\right]\right)^{2}\right\} \\
%  & =\frac{1}{\mu}tr\left(Var_{X\sim\mathcal{D}_{h}}\left(X\right)\right)+\left\Vert \mathbb{E}_{X\sim\mathcal{D}_{h}}\left[X-y\right]\right\Vert ^{2}
\end{align*}
By Lemma \ref{lem:ball_volume}, we have: 
$\mathbb{E}_{X\sim\mathcal{B}(0,\sqrt{h})}\left[||X||^2\right]=\frac{d}{d+2}h$.
Hence ~ 
 $\mathbb{E}_{X_1,...,X_\lambda\sim\mathcal{B}(0,r)}\left[f\left(\bar{X}_{(\mu)}\right)\mid f(X_{(\mu+1)})=h\right]= \frac{d}{d+2}\frac{h}{\mu}.$
% Because $f(x)=\left\Vert x-y\right\Vert ^{2}$, the support of $\mathcal{D}_{h}$
% is a $\ell_{2}$ ball centered around $y$ so we have $\mathbb{E}_{X\sim\mathcal{D}_{h}}\left[X-y\right]=0$.
% Also, we have $$tr\left(Var_{X\sim\mathcal{D}_{h}}\left(X\right)\right)=\mathbb{E}_{X\sim\mathcal{D}_{h}}\left[\left\Vert X-y\right\Vert ^{2}\right]$$ $$=\mathbb{E}_{X\sim\mathcal{B}\left(y,\sqrt{h}\right)}\left[\left\Vert X\right\Vert ^{2}\right]=\frac{d}{d+2}.h$$
% (by lemma \ref{lem:ball_volume}). Thus, we have 
% \[
% \mathbb{E}_{S}\left[f\left(\bar{X}_{(\mu)}\right)\mid f(X_{(\mu+1)})=h\right]=\frac{h}{\mu}.\frac{d}{d+2}.
% \]
\end{proof}
The result of Lemma~\ref{lem:expectation_k_avg} shows that $\mathbb{E}\left[f\left(\bar{X}_{(\mu)}\right)\mid f(X_{(\mu+1)})=h\right]$ depends linearly on $h.$ We now establish a similar dependency for 
$\mathbb{E}\left[f\left(X_{(1)}\right)\mid f(X_{(\mu+1)})=h\right]$.
\begin{lem}\label{lm4} \label{lem:expectation_1_cond} For $d>0,\ h>0,\ \lambda>\mu\geq1$,  and $f(x)=||x||^2$,
$$\mathbb{E}_{X_1,...,X_\lambda\sim\mathcal{B}(0,r)}\left[f\left(X_{(1)}\right)\mid f(X_{(\mu+1)})=h\right] = h \frac{\Gamma(\frac{d+2}{d})\Gamma(\mu+1)}{\Gamma(\mu+1+2/d)}.$$
\end{lem}
\begin{proof}
First note that using the same argument as in Lemma~\ref{lem:expectation_k_avg}, $\forall\beta\in (0,h]$:
\begin{align*}
\mathbb{P}_{X_{1}\ldots X_{\lambda}\sim\mathcal{B}(0,\sqrt{h})}&\left[f\left(X_{(1)}\right)>\beta\mid f(X_{(\mu+1)})=h\right]\\
&=\mathbb{P}_{X_{1}\ldots X_{\mu}\sim\mathcal{B}(0,\sqrt{h})}\left[f\left(X_{1}\right)>\beta,\ldots,f\left(X_{\mu}\right)>\beta\right]\\
 & =\mathbb{P}_{X\sim\mathcal{B}(0,\sqrt{h})}\left[f\left(X\right)>\beta\right]^{\mu}.
\end{align*}
Recall that the volume of a $d$-dimensional ball of radius $r$ is
proportional to $r^{d}$. Thus, we get:%\footnote{{[}formule vérifiée numériquement (section code){]}}:
\begin{align*}
\mathbb{P}_{X\sim\mathcal{B}(0,\sqrt{h})}\left[f\left(X\right)<\beta\right] & =\frac{\sqrt{\beta}^{d}}{\sqrt{h}^{d}}=\left(\frac{\beta}{h}\right)^{\frac{d}{2}}.
\end{align*}
It is known that for every positive random variable $X$, $\mathbb{E}(X)=\int_0^\infty\mathbb{P}(X>\beta)d\beta$. Therefore:
\begin{align*}
\mathbb{E}_{S}\left[f\left(X_{(1)}\right)\mid f(X_{(\mu+1)})=h\right] &=\int_0^h \mathbb{P}\left[f\left(X_{(1)}\right)>\beta\mid f(X_{(\mu+1)})=h\right]d\beta\\
&=\int_0^h \left(1-\left(\frac{\beta}{h}\right)^{\frac{d}{2}}\right)^{\mu}d\beta\\
&=h\int_0^1 \left(1-u^{\frac{d}{2}}\right)^{\mu}du\\
&=h\frac2d \int_0^1 \left(1-t\right)^{\mu} t^{2/d-1}dt
= h \frac{\Gamma(\frac{d+2}{d})\Gamma(\mu+1)}{\Gamma(\mu+1+2/d)}.
\end{align*}
To obtain the last equality, we identify the integral with the beta function of parameters $\mu+1$ and $\frac2d$.%\crc{Tried to address Olivier's comment. OT: Thx!}
%For the last equality, we use the formula\oliviernote{unclear; which formula} for the beta function of parameters $\mu+1$ and $\frac2d$.
\end{proof}
We now directly compute $\mathbb{E}_{X_1,...,X_\lambda\sim\mathcal{B}(0,r)}\left[ f(X_{(1)})\right]$.
%\crc{I removed the notation $S$ that was used below and in a number of places throughout the paper. I think we do not need it (and we haven't defined it anyway). It is probably better to be explicit about the variables we're taking the expectation on.}
\begin{lem}\label{lem:expectation_1}\label{lm5}
For all $d>0$, $\lambda>0$ and $r>0$:
 $$\mathbb{E}_{X_1,...,X_\lambda\sim\mathcal{B}(0,r)}\left[ f(X_{(1)})\right]=r^2\frac{\Gamma(\frac{d+2}{d})\Gamma(\lambda+1)}{\Gamma(\lambda+1+2/d)}.$$
\end{lem}
\begin{proof}
As in Lemma~\ref{lem:expectation_1_cond}, we have for any 
$\beta \in (0,r^2]$:
\begin{align*}
\mathbb{P}_{X_{1}\ldots X_{\lambda}\sim\mathcal{B}(0,r)}\left[f\left(X_{(1)}\right)>\beta\right] &=\mathbb{P}_{X_{1}\ldots X_{\lambda}\sim\mathcal{B}(0,r)}\left[f\left(X_{1}\right)>\beta,...,f\left(X_{\lambda}\right)>\beta\right]\\
&=\mathbb{P}_{X\sim\mathcal{B}(0,r)}\left[f\left(X\right)>\beta\right]^\lambda\\
&= \left(\frac{\sqrt{\beta}}{r}\right)^{d}.
\end{align*}
The result then follows by reasoning as in the proof of Lemma~\ref{lem:expectation_1_cond}. 
\end{proof}
By combining the results above, we obtain the exact formula for 
%We now compute, from previous lemmas, the exact formula for 
$\mathbb{E}\left[ f(\bar X_{(\mu)})\right]$.
\begin{thm}\label{thm:formula_k_best}
For all $d>0$, $r>0$ and $\lambda >\mu \ge 1$:
$$\mathbb{E}_{X_{1}\ldots X_{\lambda}\sim\mathcal{B}(0,r)}\left[ f(\bar X_{(\mu)})\right]=\frac{r^2d\times\Gamma(\lambda+1)\Gamma(\mu+1+2/d)}{\mu(d+2)\Gamma(\mu+1)\Gamma(\lambda+1+2/d)}.$$
\end{thm}
\begin{proof}
%\crc{If we draw $\mu=\lambda$ points, how is $X_{\mu+1}=X_{(\lambda+1)}$ defined? OT: looks like this is addressed now.}
% \begin{align*}
% \mathbb{E}_{X_{1}\ldots X_{\lambda}\sim\mathcal{B}(0,r)}\left[ f(\bar X_{(\mu)})\right]&=\mathbb{E}\left[ f(\bar X_{(\mu)})\mid f\left(X_{(\mu+1)}=h\right)\right]\mathds{1}\left(f\left(X_{(\mu+1)}\right)=h\right)\\
% &=\frac{h}{\mu}\frac{d}{d+2}\mathds{1}\left(f\left(X_{(\mu+1)}\right)=h\right)\\
% %&=\frac{1}{\mu}\frac{d}{d+2}\mathbb{E}\left[ f\left(X_{(\mu+1)}\right)\right]\\
% &=\frac{1}{\mu}\frac{d}{d+2} \frac{\Gamma(\mu+1+2/d)}{\Gamma(\mu+1)\Gamma(1+2/d)}\mathbb{E}\left[ f( X_{(1)})\mid f\left(X_{(\mu+1)}\right)=h\right]\mathds{1}\left(f\left(X_{(\mu+1)}\right)=h\right)\\
% &=\frac{1}{\mu}\frac{d}{d+2} \frac{\Gamma(\mu+1+2/d)}{\Gamma(\mu+1)\Gamma(1+2/d)}\mathbb{E}\left[ f( X_{(1)})\right]\\
% &=\frac{r^2d\times\Gamma(\lambda+1)\Gamma(\mu+1+2/d)}{\mu(d+2)\Gamma(\mu+1)\Gamma(\lambda+1+2/d)}
% \end{align*}
The proof follows by applying our various lemmas and integrating over all possible values for $h$. We have:
\begin{align*}
&\mathbb{E}_{X_{1}\ldots X_{\lambda}\sim\mathcal{B}(0,r)}\left[ f(\bar X_{(\mu)})\right]\\
&=\mathbb{E}\left[\mathbb{E}\left[ f(\bar X_{(\mu)})\mid f\left(X_{(\mu+1)}\right)\right]\right]\\
&=\frac{1}{\mu}\frac{d}{d+2}\mathbb{E}\left[ f\left(X_{(\mu+1)}\right)\right]\mbox{  by Lemma \ref{lm3}}\\
&=\frac{1}{\mu}\frac{d}{d+2} \frac{\Gamma(\mu+1+2/d)}{\Gamma(\mu+1)\Gamma(\frac{d+2}{d})}\mathbb{E}\left[\mathbb{E}\left[ f( X_{(1)})\mid f\left(X_{(\mu+1)}\right)\right]\right]\mbox{ by Lemma \ref{lm4}}\\
&=\frac{1}{\mu}\frac{d}{d+2} \frac{\Gamma(\mu+1+2/d)}{\Gamma(\mu+1)\Gamma(\frac{d+2}{d})}\mathbb{E}\left[ f( X_{(1)})\right]\\
&=\frac{r^2d\times\Gamma(\lambda+1)\Gamma(\mu+1+2/d)}{\mu(d+2)\Gamma(\mu+1)\Gamma(\lambda+1+2/d)}\mbox{ by Lemma~\ref{lm5}}.
\end{align*}
\end{proof}
We have checked experimentally the result of Theorem~\ref{thm:noncentered} (see Figure~\ref{fig:exp_th_c}): the result of Theorem~\ref{thm:k_avg_best_vs_1_best} follows from Theorem~\ref{thm:noncentered} since for $d\geq2$, $\lambda$ and $r$ fixed, $\mathbb{E}\left[ f(\bar X_{(\mu)})\right]$ is strictly decreasing in $\mu$. In addition, we can obtain asymptotic progress rates:
\begin{cor}\label{cor:equiv}
Consider $d>0$.
When $\lambda\rightarrow\infty$, we have
$$\mathbb{E}_{X_{1}\ldots X_{\lambda}\sim\mathcal{B}(0,r)}\left[ f(\bar X_{(\mu)})\right]\sim \lambda^{-\frac{2}{d}}\frac{r^2d\times\Gamma(\mu+1+2/d)}{\mu(d+2)\Gamma(\mu+1)},$$
$$\mbox{while if $\lambda\rightarrow \infty$ and $\mu(\lambda)\rightarrow\infty$,~}
\mathbb{E}_{X_{1}\ldots X_{\lambda}\sim\mathcal{B}(0,r)}\left[ f(\bar X_{(\mu(\lambda))})\right]\sim r^2\frac{ d}{d+2}\frac{\mu(\lambda)^{\frac{2}{d}-1}}{\lambda^{\frac{2}{d}}}.$$
As a result, $\forall c\in(0,1)$, $\mathbb{E}\left(f(\bar X_{(\left\lfloor c\lambda\right\rfloor)})\right)\in \Theta\left(\frac1\lambda\right)$ and $\mathbb{E}\left(f(X_{(1)})\right)\in \Theta\left(\frac{1}{\lambda^{2/d}}\right)$.
\end{cor}
\begin{proof}
We recall the Stirling equivalent formula for the gamma function: when $z\rightarrow\infty$,
$$\Gamma(z) = \sqrt{\frac{2\pi}{z}}\left(\frac{z}{e}\right)^z\left(1+O\left(\frac{1}{z}\right)\right).$$
Using this approximation, we get the expected results.
\end{proof}
This result shows that by keeping a single parent, we lose more than a constant factor: the progress rate is significantly impacted. Therefore it is preferable to use more than one parent.%Hence we might \emph{not} use a single parent, since the rate is worse. 
\def\rere{
\subsection{Ratio mean $/$ min when sampling in $B(y,r)$}\label{ouaip}
Consider
$mean_\mu$ the expectation of $g(\bar x_\mu)$
and $min_\mu$ the expectation of $g(x_{1})$.
Using previous results,
\begin{itemize}
\item $mean_h\leq\frac h\mu \frac{d}{d+2}$.
\item $min_h\geq h\frac{\Gamma(\frac{d+2}{d})\Gamma(\mu+1)}{\Gamma(\mu+1+\frac2d)}$.
\end{itemize}
Then 
$$\frac{mean_h}{min_h}=\frac{d\Gamma(\mu+1+2/d)}{\mu(d+2)\Gamma(\frac{d+2}d)\Gamma(\mu+1)}\leq \mu^{2/d-1}(1+o(1)),$$
where the $o(1)$ depends on $d$ only.
}
\subsection{Convergence when the sampling is not centered on the optimum}\label{harder}
So far we treated the case where the center of the distribution and the optimum are the same. We now assume that we sample from the distribution $\mathcal{B}(0,r)$ and that the function $f$ is $f(x) = ||x-y||^2$ with $||y||\leq r$. We define $\epsilon = \frac{||y||}{r}$. 
%Let $h=f(x_{(k+1)})$.
\begin{lem}\label{lem:bin} Let $r>0$, $d>0$, $\lambda>\mu\geq1$, we have:
 $$\mathbb{P}_{X_{1}\ldots X_{\lambda}\sim\mathcal{B}(0,r)}(f(X_{(\mu+1)})>(1-\epsilon)^2r^2)= \mathbb{P}_{U\sim B\left(\lambda,(1-\epsilon)^{d}\right)}\left(U\leq \mu\right),$$
where $B(\lambda,p)$ is a binomial law of parameters $\lambda$ and $p$.
\end{lem}
\begin{proof}
We have
$ f(X_{(\mu+1)})>(1-\epsilon)r\iff \sum_{i=1}^{\lambda} \mathds{1}_{\{f(X_{i})\leq(1-\epsilon)^2r^2\}}\leq \mu$
since $\mathds{1}_{\{f(X_{i})\leq(1-\epsilon)^{2}r^{2}\}}$ are independent Bernoulli variables of parameter $(1-\epsilon)^{d}$,  hence the result.
\end{proof}
Using Lemma~\ref{lem:bin}, we now get lower and upper bounds on $\mathbb{E}\left[f\left(X_{(\mu+1)}\right)\right]$: %as follows.
% \def\zehoeffdingbound{\exp\left(-2n\left(\left(\frac{r-||y||}r\right)^d-k/n\right)^2\right)}
% \begin{lem}\label{totoro}
% Assume $r>||y||$. Then
% with $\alpha=(r-||y||)^2$,
% $P(||\bar x_\mu||>r-||y||)\leq \zehoeffdingbound$.
% \end{lem}
% \begin{proof}
% Application of Hoeffding's bound to the sum of $\lambda$ independent identically distributed random variables with values in $[0,1]$.
% \end{proof}
% Define
% \begin{eqnarray*}
% R&=&\E f(\bar X_\mu)\mbox{ with sampling in $B(0,r)$}\\
% R'&=&\E f(\bar X_\mu)\mbox{ with sampling in $B(y,r)$}\\
% R_h&=&\E f(\bar X_\mu|f(X_{(\mu)})\leq h)\mbox{ with sampling in $B(0,r)$}\\
% R'_h&=&\E f(\bar X_\mu|f(X_{(\mu)}\leq h)\mbox{ with sampling in $B(y,r)$}
% \end{eqnarray*}
% \begin{lem}\label{ttt}
%  If $r>||y||$, then
%  there exists $e\in[0,1]$ such that
%  $R= R'+2e(2r)^2p=R'+8r^2p$,
%  where $p=P(||x||>r-||y||)$ for $x$ uniformly drawn in $B(0,r)$.
% \end{lem}
% \begin{proof}
% By Lemma\ref{totoro},
% \begin{eqnarray*}
% |R-R_h|&\leq & (2r)^2p\\
% |R'-R'_h|&\leq & (2r)^2p\\
% \end{eqnarray*}
% In addition,
% $R_h=R'_h$, hence the expected result.
% \end{proof}
% Then Theorem \ref{barx}
% becomes as follows:
% \begin{thm}
% When sampling with $B(0,r)$, if $r>||y||$,
% $$\mathbb{E}_S\left[ f(\bar X_{(\mu)})\right]= \frac{r^2d\times\Gamma(n+1)\Gamma(k+1+2/d)}{k(d+2)\Gamma(k+1)\Gamma(n+1+2/d)}$$
% $$+8r2e\zehoeffdingbound$$
% for some $e\in [0,1]$.
% \end{thm}
\begin{thm}\label{thm:noncentered} 
Consider $d>0$, $r>0$, $\lambda>\mu \ge 1$.
The expected value of $f(\bar{X}_{(\mu)})$ satisfies both
\begin{align*}
\mathbb{E}_{X_{1}\ldots X_{\lambda}\sim\mathcal{B}(0,r)}\left[f(\bar{X}_{(\mu)})\right]\leq &4r^2\mathbb{P}_{U\sim B\left(\lambda,(1-\epsilon)^{d}\right)}\left(U\leq \mu\right)\\
&+\frac{r^2 d\times\Gamma(\lambda+1)\Gamma(\mu+1+2/d)}{\mu(d+2)\Gamma(\mu+1)\Gamma(\lambda+1+2/d)}
\end{align*}
$$\mbox{ and ~~}\mathbb{E}_{X_{1}\ldots X_{\lambda}\sim\mathcal{B}(0,r)}\left[f(\bar{X}_{(\mu)})\right]\geq \frac{r^2 d\times\Gamma(\lambda+1)\Gamma(\mu+1+2/d)}{\mu(d+2)\Gamma(\mu+1)\Gamma(\lambda+1+2/d)}.$$
\end{thm}
\begin{proof}
\begin{align*}
\mathbb{E}\left[f(\bar{X}_{(\mu)})\right]  
&= \mathbb{E}\left(f(\bar{X}_{(\mu)})|f(X_{(\mu+1)})\geq (1-\epsilon)^{2}r^{2}\right)\mathbb{P}\left(f(X_{(\mu+1)})\geq (1-\epsilon)^{2}r^{2}\right)\\
&+\mathbb{E}\left(f(\bar{X}_{(\mu)})|f(X_{(\mu+1)})< (1-\epsilon)^{2}r^{2}\right)\mathbb{P}\left(f(X_{(\mu+1)})<(1-\epsilon)^2r^2\right).
 \end{align*}
In this Bayes decomposition, we can bound the various terms as follows:  \begin{eqnarray*}
\mathbb{E}\left(f(\bar{X}_{(\mu)})|f(X_{(\mu+1)})\geq (1-\epsilon)^2r^2\right)&\leq& 4r^2,\\
\mathbb{P}\left(f(X_{(\mu+1)})\geq (1-\epsilon)^2r^2\right)&\leq &1,\\ \mathbb{E}\left[f(\bar{X}_{(\mu)})|f(X_{(\mu+1)})< (1-\epsilon)^2r^2\right]&\leq&\frac{r^2 d\times\Gamma(\lambda+1)\Gamma(\mu+1+2/d)}{\mu(d+2)\Gamma(\mu+1)\Gamma(\lambda+1+2/d)}.
\end{eqnarray*}
Combining these equations yields the first (upper) bound.
%\otc{Combining these equations yields the first bound.}
The second (lower) bound is deduced from the centered case (i.e. when the distribution is centered on the optimum) as in the previous section.
%\otc{: this case provides the lower bound.}
\end{proof}
Figure~\ref{fig:exp_th_nc} gives an illustration of the bounds. Until $\mu\simeq (1-\epsilon)^{d}\lambda$, the centered and non centered case coincide when $\lambda\rightarrow\infty$: in this case, we can have a more precise asymptotic result for the choice of $\mu$.
\begin{thm}\label{thm:asymp}
Consider $d>0$, $r>0$ and $y\in\R^d$.
%The optimal choice for $\mu$ when $n\rightarrow\infty$ is $\frac{\mu}{n}\rightarrow(1-||y||/R)^d$.
Let $\epsilon = \frac{||y||}{r}\in [0,1)$ and $f(x) = ||x-y||^2$. When using $\mu=\lfloor c\lambda\rfloor$ with $0<c<(1-\epsilon)^{d}$,
we get as $\lambda\to\infty$, for a fixed $d$,
$$\mathbb{E}_{X_1...X_\lambda\sim\mathcal{B}(0,r)}\left[ f(\bar X_{(\mu)})\right]= \frac{dr^2c^{2/d-1}}{(d+2)\lambda}+o\left(\frac1\lambda\right).$$
\end{thm}
\begin{proof}
Let $\mu_\lambda=\lfloor c\lambda\rfloor$ with $0<c<(1-\epsilon)^{d}$. We  immediately have from Hoeffding's concentration inequality: $$\mathbb{P}_{U\sim B\left(\lambda,(1-\epsilon)^d\right)}\left(U\leq \mu_\lambda\right)\in o(\frac{1}{\lambda})$$ when $\lambda \rightarrow \infty$.
%From Corollary~\ref{cor:equiv}, we have: $$\frac{r^2 d\times\Gamma(\lambda+1)\Gamma(\mu+1+2/d)}{\mu(d+2)\Gamma(\mu+1)\Gamma(\lambda+1+2/d)}\sim \frac{r^2c^{2/d-1}}{(d+2)\lambda}$$
From Corollary~\ref{cor:equiv}, we also get: $$\frac{r^2 d\times\Gamma(\lambda+1)\Gamma(\mu_\lambda+1+2/d)}{\mu_\lambda(d+2)\Gamma(\mu_\lambda+1)\Gamma(\lambda+1+2/d)}\sim \frac{d\,r^2 c^{2/d-1}}{(d+2)\lambda}.$$
Using the inequalities of Theorem~\ref{thm:noncentered}, we obtain the desired result.
\end{proof}
The result of Theorem~\ref{thm:asymp} shows that a convergence rate $O(\lambda^{-1})$ can be attained for the $\mu$-best approach with $\mu>1$. The rate for $\mu=1$ is $\Theta(\lambda^{-2/d})$, proving that the $\mu$-best approach leads asymptotically to a better estimation of the optimum. 
    %\item $\Theta(n^{-1})$ for all algorithms, including not one-shot algorithms\cite{fabian,chen1988}, for a wider range of objective functions.
%\end{itemize}
If we consider the problem $\min_\mu\max_{y:||y||\leq \epsilon r}\mathbb{E}\left[f_y(\bar{X}_{(\mu)})\right]$ with $f_y$ the objective function $x\mapsto ||x-y||^2$, then $\mu=\lfloor c\lambda\rfloor$ with $0<c<(1-\epsilon)^{d}$ achieves the  $O\left(\lambda^{-1}\right)$ progress rate. 

All the results we proved in this section are easily extendable to strongly convex quadratic functions. For larger class of functions, it is  less immediate, and left as future work.
% \todo{Normal with reweighting! Iterative case (should be easy in the case of the sphere) and get rates, which class of function to generalize??}
% % \section{General case}
% % General hypotheses on $f$: if the exists $m,M>0$ s.t. $m I \leq \nabla^2f(x)\leq M I$ then the results holds assymptotically (for the simple regret of $f$, the second ca be) . (the idea is that $m I \leq \nabla^2f(x)$ ensures the order of the the evaluated values is the same than the ellipsoid is the same 
% \section{Generalization to $\mathcal{C}^2$-strongly convex functions}
% Let first recall what is a strongly convex function. A $\mathcal{C}^2$-function $f$ is said to $\mu$-strongly convex iif for all $x$, $\nabla^2 f(x)\succeq \mu I$. It is a well known fact that $f$ has a unique optimum $x^\star$. Around this optimum we have: $$f(x) = f(x^\star)+\left(x-x^\star\right)^T\nabla^2 f(x^*)\left(x-x^\star\right) +o\left(||x-x^\star||^2\right)$$.
% Close to $x^*$ the order of $f$ is hence the same as $x\mapsto \left(x-x^\star\right)^T\nabla^2 f(x^*)\left(x-x^\star\right)$. In a more formal way, $\exists \epsilon>0,\text{ }\forall{x,y}\in B(x^\star,\epsilon)$ s.t. :
% $$f(x)\leq f(y)\iff \left(x-x^\star\right)^T\nabla^2 f(x^*)\left(x-x^\star\right)\leq  \left(y-x^\star\right)^T\nabla^2 f(x^*)\left(y-x^\star\right) $$
\subsection{Using quasi-convexity}\label{qc}
\begin{figure}[t]
    \centering
    \includegraphics[width=.45\textwidth]{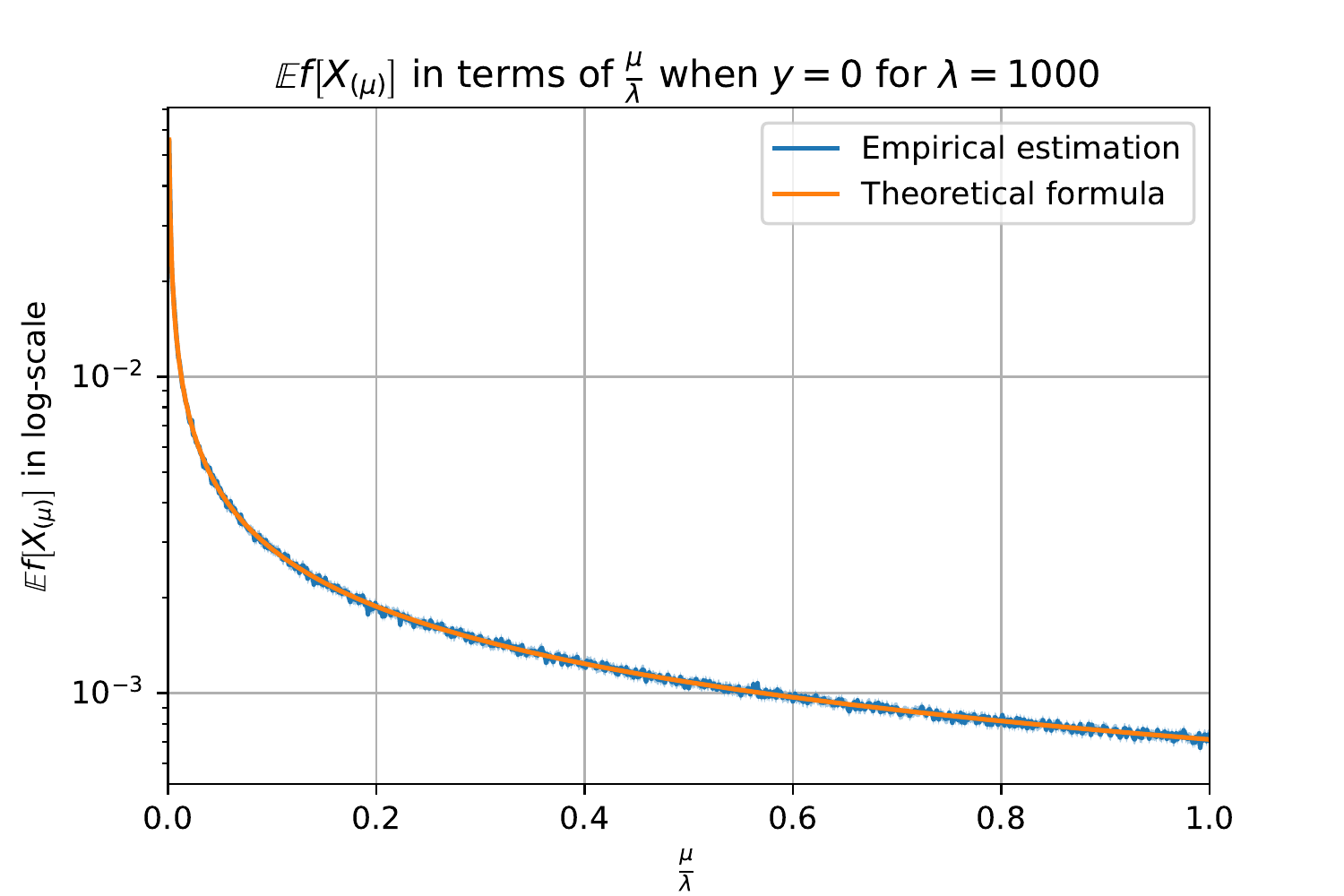}
    \caption{Centered case: validation of the theoretical formula for $\mathbb{E}_{X_1...X_\lambda\sim\mathcal{B}(0,r)}\left[ f(\bar X_{(\mu)})\right]$ when $y=0$ from Theorem~\ref{thm:formula_k_best} for $d=5$, $\lambda=1000$ and $R=1$. $1000$  samples have been drawn to estimate the expectation. \otc{The two curves overlap, showing agreement between theory and practice.}}
    \label{fig:exp_th_c}
\end{figure}

The method above was designed for the sphere function, yet its adaptation to other quadratic convex functions is straightforward. On the other hand, our reasoning might break down when applied to 
multimodal functions.
%But we can guess that it can be detrimental in multimodal functions.
We thus consider an adaptive strategy to define $\mu$.  A desirable property to a $\mu$-best approach is that the level-sets of the functions are convex.  \otc{A simple workaround is to choose $\mu$ maximal such that there is a quasi-convex function which is identical to $f$ on $\{X_{(1)},\dots,X_{(\mu)}\}$.}
%\laurentnote{This part is not very clear}
%\crc{I agree with Laurent. I tried to reformulate what I think is more a discussion than a rigorous argument.}
%{\bf{Preserving theoretical guarantees:}}
If the objective function is quasi-convex on the convex hull of $\{X_{(1)},\dots,X_{(\tilde{\mu})}\}$ with $\tilde{\mu} \le \lambda$, then: for any $i \le \tilde{\mu}$,
$X_{(i)}$ is on the frontier (denoted $\partial$) of the convex hull of $\{X_{(1)},\dots,X_{(i)}\}$ and the value $$h=\max \left\{ i\in [1,\lambda], \forall j\leq i, X_{(j)}\in \partial\left[\texttt{ConvexHull}(X_{(1)},\dots,X_{(j)})\right]\right\}$$ verifies $h \geq \tilde \mu$ so that $\mu=\min(h,\tilde\mu)$ is actually equal to $\tilde \mu$.
%If the objective function is quasi-convex, then:
% \begin{itemize}
% \item for any $i \le \tilde{\mu}$,
% $X_{(i)}$ is on the frontier (denoted $\partial$) of the convex hull of $\{X_{(1)},\dots,X_{(i)}\}$.
% %\item By definition, $X_{(i)}$ is on the frontier (denoted $\partial$) of the convex hull of $\{X_{(1)},\dots,X_{(i)}\}$.
% %\item Therefore, for all $\tilde \mu\leq \lambda$, the value $$h=\max \left\{ i\in [1,\lambda], \forall j\leq i, X_{(j)}\in \partial\left[\texttt{ConvexHull}(X_{(1)},\dots,X_{(j)})\right]\right\}$$ verifies $h\geq \lambda\geq \tilde \mu$ so that $\mu=\min(h,\tilde\mu)$ is actually equal to $\tilde \mu$.
% \item Thus, the value $$h=\max \left\{ i\in [1,\lambda], \forall j\leq i, X_{(j)}\in \partial\left[\texttt{ConvexHull}(X_{(1)},\dots,X_{(j)})\right]\right\}$$ verifies $h \geq \tilde \mu$ so that $\mu=\min(h,\tilde\mu)$ is actually equal to $\tilde \mu$.
% \end{itemize}
\otc{As a result:
\begin{itemize}
    \item in the case of the sphere function, or any quasi-convex function, if we set $\tilde \mu=\lfloor \lambda (1-\epsilon)^d\rfloor$, using $\mu=\min(h,\tilde \mu)$ leads to the same value of $\mu=\tilde \mu = \lfloor \lambda (1-\epsilon)^d\rfloor$. In particular, we preserve the theoretical guarantees of the previous sections for the sphere function $x\mapsto ||x-y||^2$.
    \item if the objective function is not quasi-convex, we can still compute the quantity $h$ defined above, but we might get a $\mu$ smaller than $\tilde \mu$. However, this strategy remains meaningful at it prevents from keeping too many points when the function is  ``highly'' non-quasi-convex.
\end{itemize}
}
% \begin{figure}[t]
%     \centering
%     \includegraphics[width=.45\textwidth]{imgs/fig_y=0.pdf}
%     \caption{Centered case: validation of the theoretical formula for $\mathbb{E}_{X_1...X_\lambda\sim\mathcal{B}(0,r)}\left[ f(\bar X_{(\mu)})\right]$ when $y=0$ from Theorem~\ref{thm:formula_k_best} for $d=5$, $\lambda=1000$ and $R=1$. $10000$  samples have been drawn to estimate the expectation. \otc{The theory matches the practice.}}
%     \label{fig:exp_th_c}
% \end{figure}

%We observe that, with this modified version using quasi-convexity, we will prevent $\mu$ from becoming too large. %\oliviernote{Keep $h$ for consistency with other parts of the doc}
% when non-convexity is visible we will prevent $\mu$ from being too large.
\section{Experiments}
\begin{figure}[t]
    \centering
    \includegraphics[width=.45\textwidth]{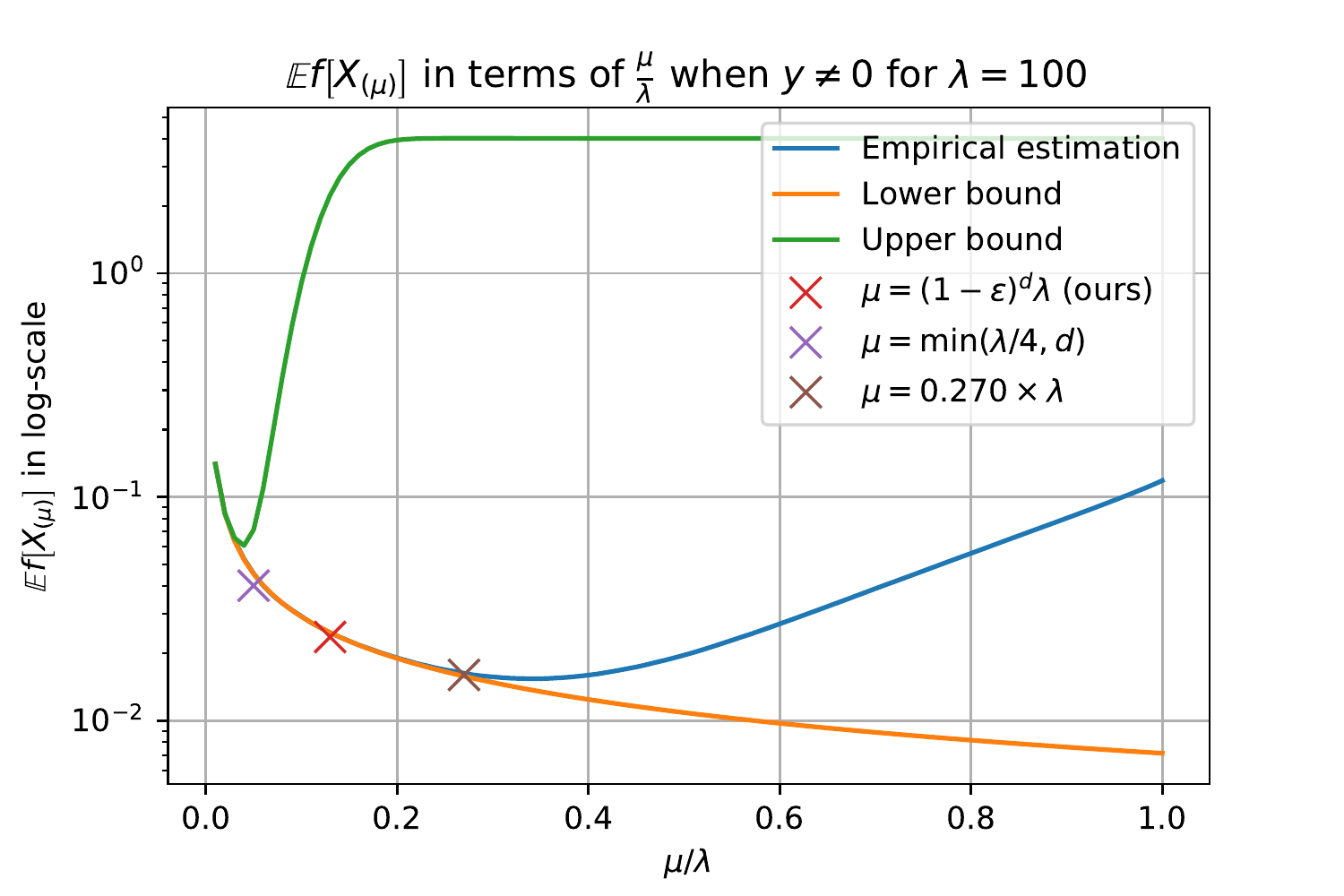}
    \includegraphics[width=.45\textwidth]{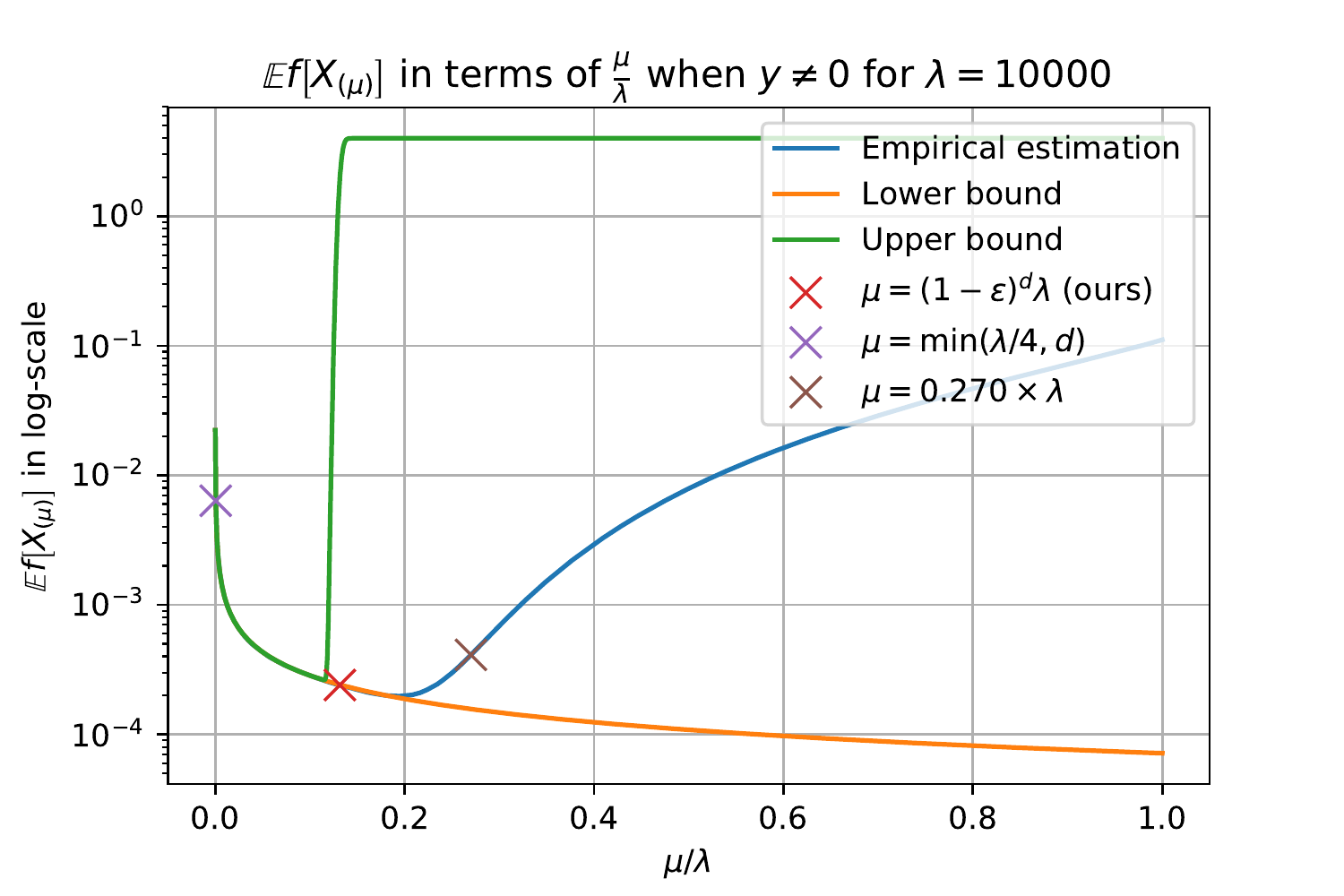}
    \caption{Non centered  case: validation of the theoretical bounds for $\mathbb{E}_{X_1...X_\lambda\sim\mathcal{B}(0,r)}\left[ f(\bar X_{(\mu)})\right]$ when $||y||=\frac{R}{3}$ (i.e. $\epsilon=\frac{1}{3}$) from Theorem~\ref{thm:noncentered} for $d=5$ and $R=1$. We implemented $\lambda=100$ and $\lambda=10000$. $10000$  samples have been drawn to estimate the expectation. We see that such a value for $\mu$ is a good approximation of the minimum of the empirical values\otc{: we can thus recommend $\mu=\lfloor \lambda (1-\epsilon)^d\rfloor$ when $\lambda\rightarrow\infty$.} We also added some classical choices of values for $\mu$ from literature: when $\lambda\rightarrow\infty$, our method performs the best.}
    \label{fig:exp_th_nc}
\end{figure}

%\subsection{Validation of the theoretical results}
%\otc{the theoretical formula from Theorem~\ref{thm:formula_k_best} and its empirical estimation: we note that the results coincide}
To validate our theoretical findings, we first compare the formulas obtained in Theorems~\ref{thm:formula_k_best} and~\ref{thm:noncentered} with their empirical estimates. We then perform larger scale experiments in a one-shot optimization setting. 

\subsection{Experimental validation of theoretical formulas}
Figure~\ref{fig:exp_th_c} compares the theoretical formula from Theorem~\ref{thm:formula_k_best} and its empirical estimation: we note that the results coincide and validate our formula. Moreover, the plot confirms that taking the $\mu$-best points leads to a lower regret than the $1$-best approach.

We also compare in Figure~\ref{fig:exp_th_nc} the theoretical bounds from Theorem~\ref{thm:noncentered} with their empirical estimates. We remark that for $\mu\leq (1-\epsilon)^d\lambda$ the convergence of the two bounds to $\mathbb{E}(f(\bar{X}_{(\mu)}))$ is fast. There exists a transition phase around $\mu\simeq (1-\epsilon)^d\lambda$ on which the regret is reaching a minimum: thus, one needs to choose $\mu$ both small enough to reduce bias and large enough to reduce variance. We compared to other empirically estimated values for $\mu$ from~\cite{escompr,HAN,cmsa}. It turns out that if the population is large, our formula for $\mu$ leads to a smaller regret. Note that our strategy assumes that $\epsilon$ is known, which is not the case in practice. It is interesting to note that if the center of the distribution and the optimum are close (i.e. $\epsilon$ is small), one can choose a larger $\mu$ to get  a lower variance on the estimator of the optimum.

\subsection{One-shot optimization in Nevergrad}

\begin{figure}[t]
\centering
\includegraphics[trim={0 0 0 2}, clip, width=.47\textwidth]{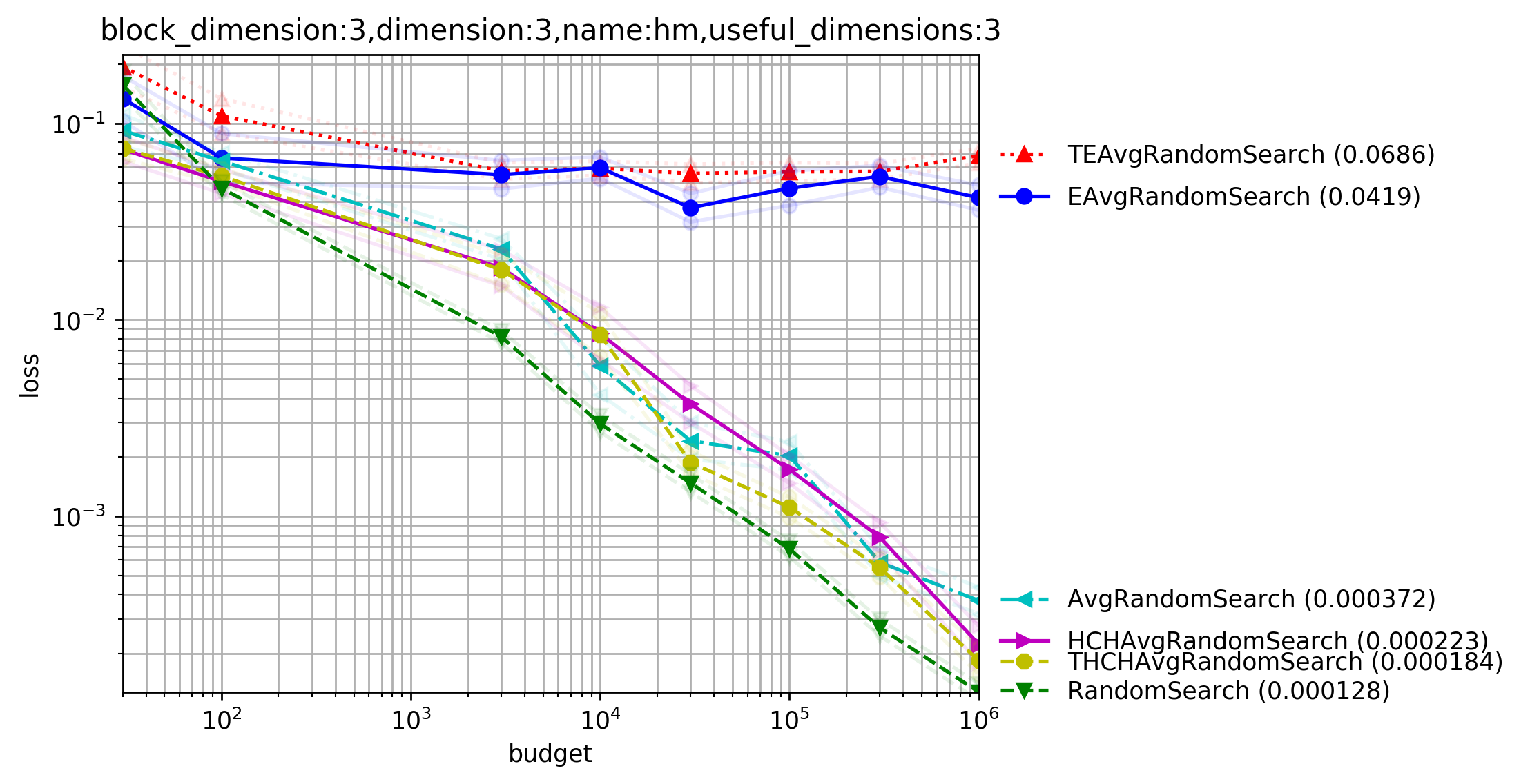}
\includegraphics[trim={0 0 0 2}, clip, width=.47\textwidth]{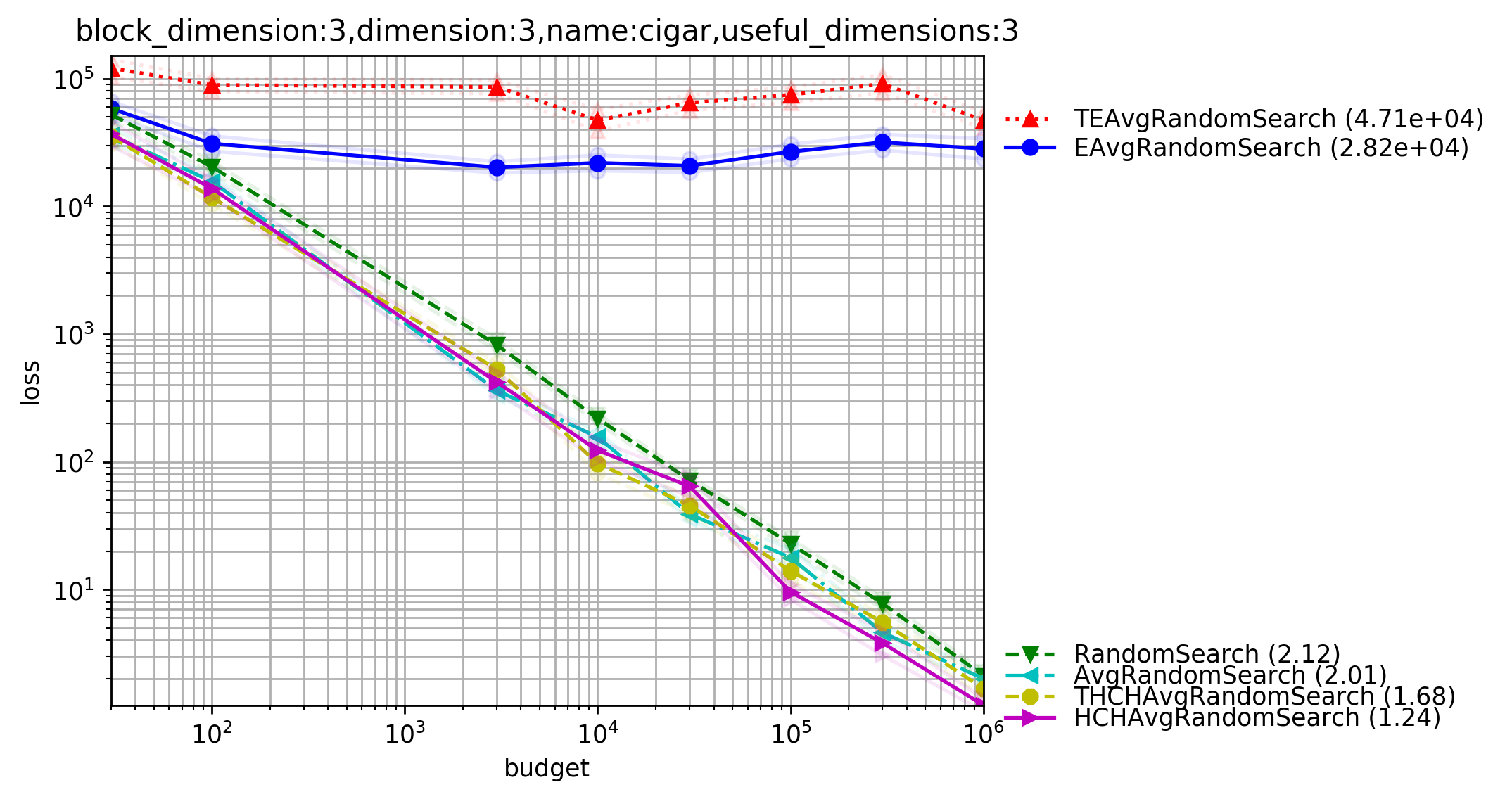}
\includegraphics[trim={0 0 0 20}, clip, width=.47\textwidth]{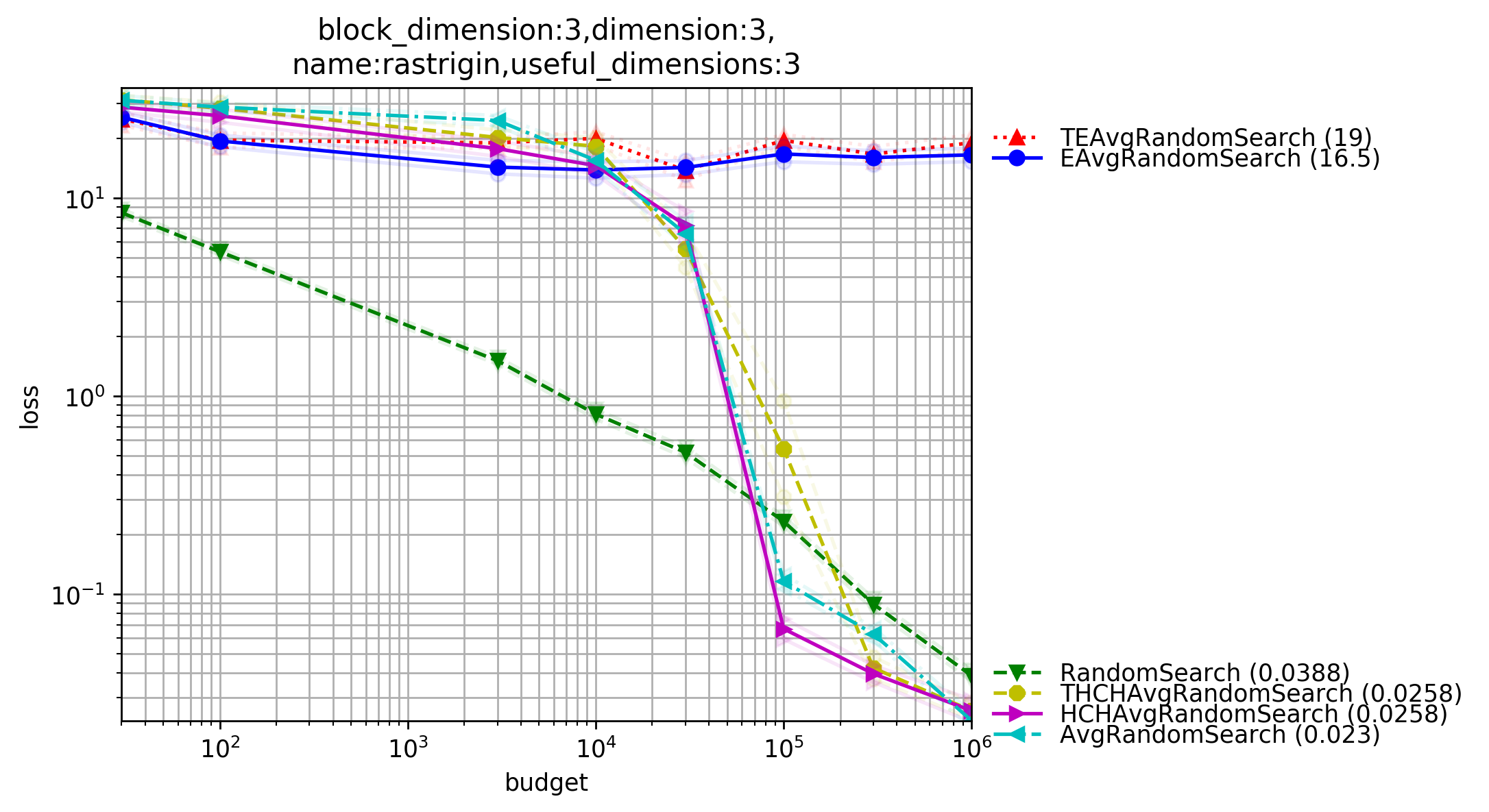}
\includegraphics[trim={0 0 0 20}, clip, width=.47\textwidth]{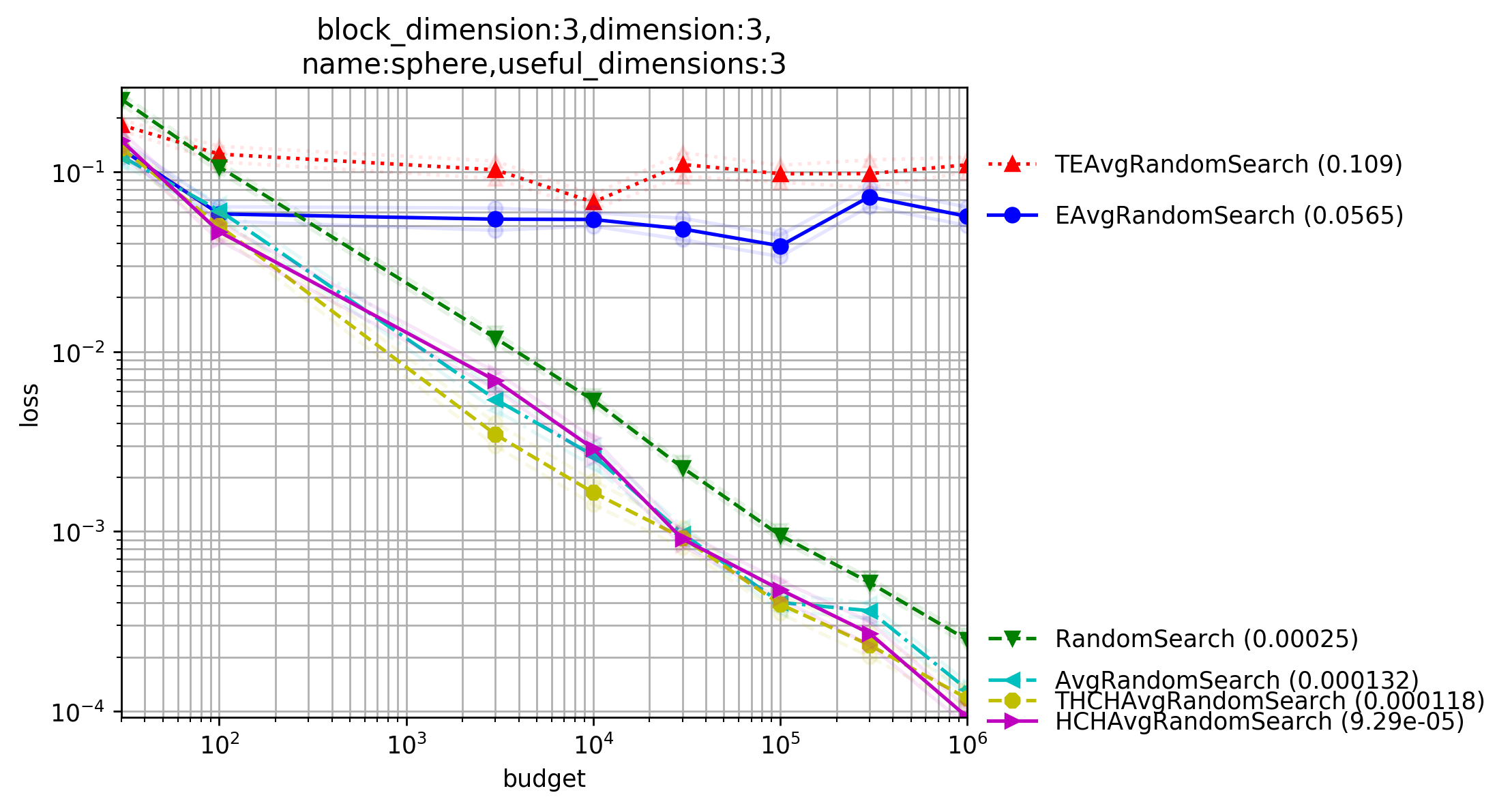}\caption{\label{nevergrad}Experimental curves comparing various methods for choosing $\mu$ as a function of $\lambda$ in dimension $3$. Standard deviations are shown by lighter lines (close to the average lines). Each x-axis value is computed independently. Our proposed formulas \texttt{HCHAvg} and \texttt{THCHAvg} perform well overall. See Fig. \ref{nevergrad25} for results in dimension 25.}
\end{figure}

In this section we test different formulas and variants for the choice of $\mu$ for a larger scale of experiments in the one-shot setting. Equations \ref{eqbeg}-\ref{eqlast} present the different formulas for $\mu$ used in our comparison.
\begin{figure}[t]
\centering
\includegraphics[trim={0 0 0 0}, clip, width=.48\textwidth]{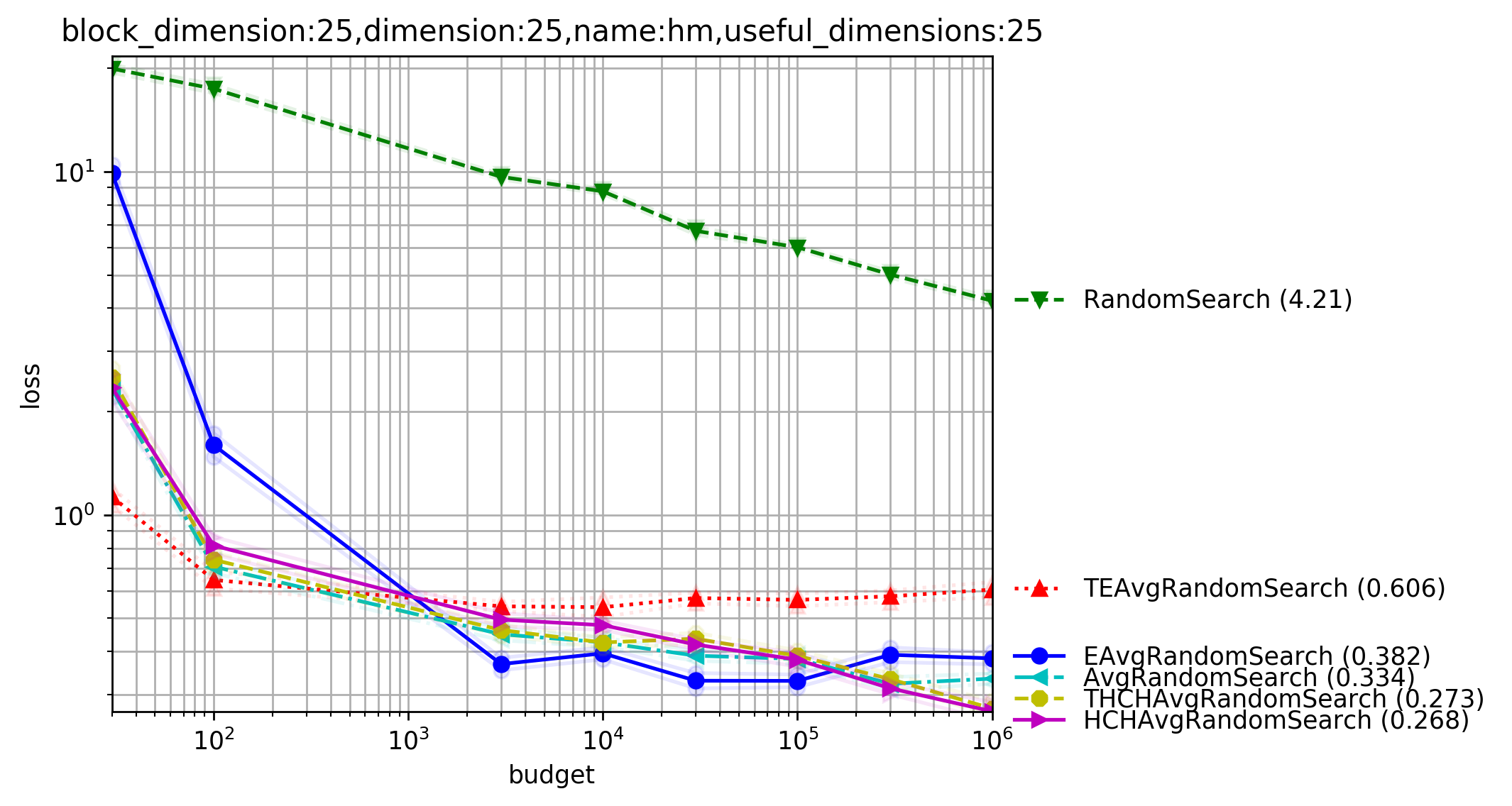}
\includegraphics[trim={0 0 0 20}, clip, width=.48\textwidth]{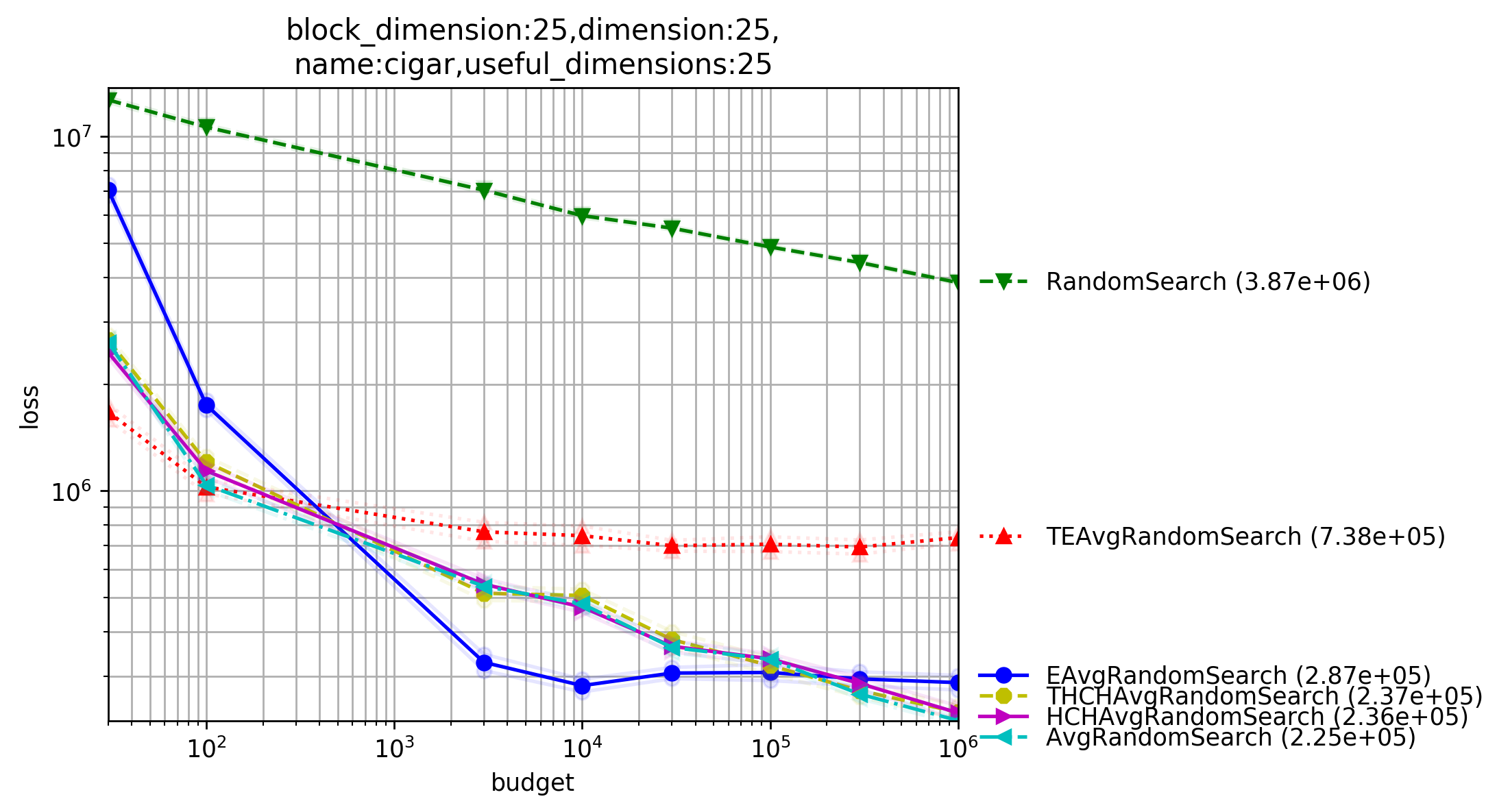}
\includegraphics[trim={0 0 0 20}, clip, width=.48\textwidth]{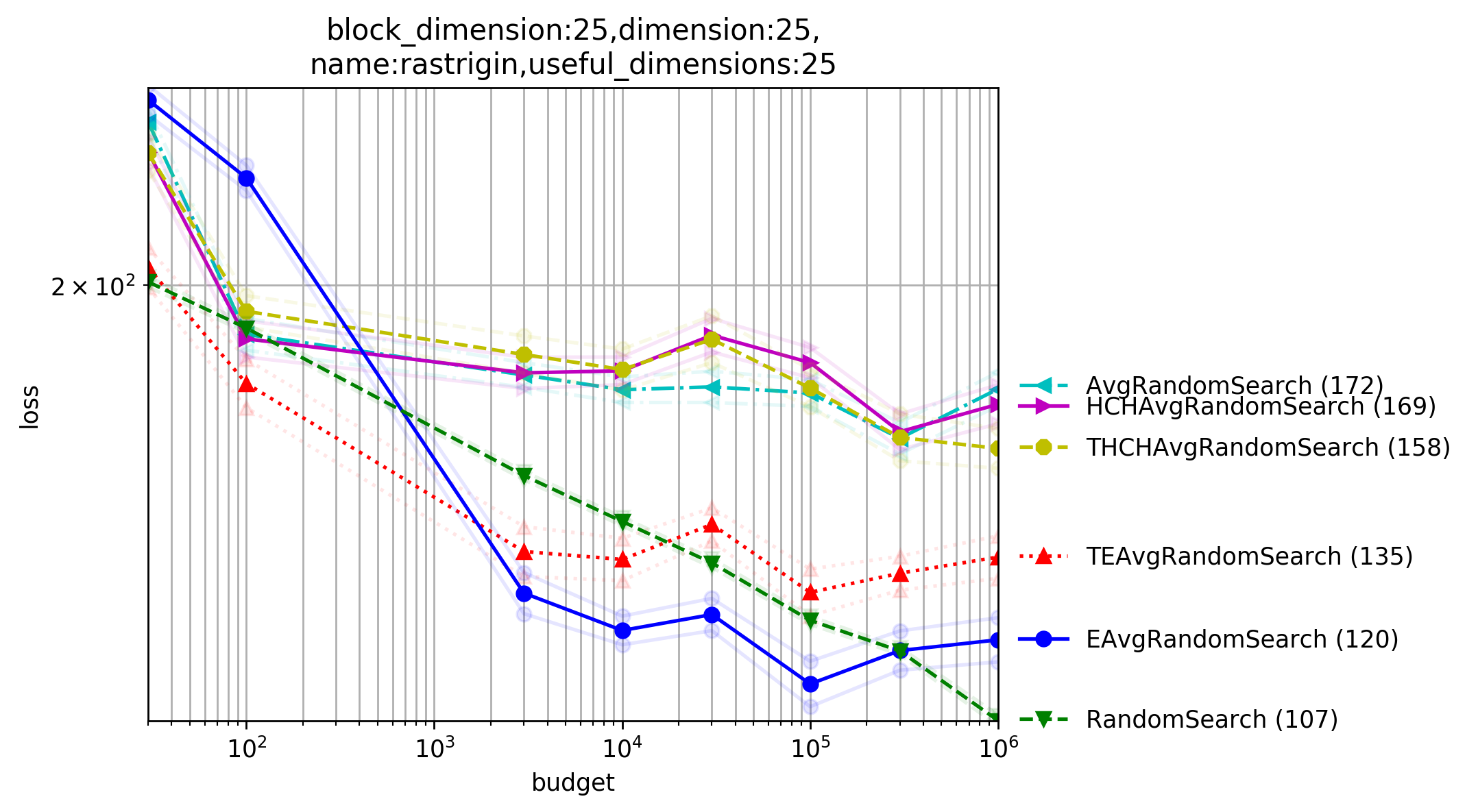}
\includegraphics[trim={0 0 0 20}, clip, width=.48\textwidth]{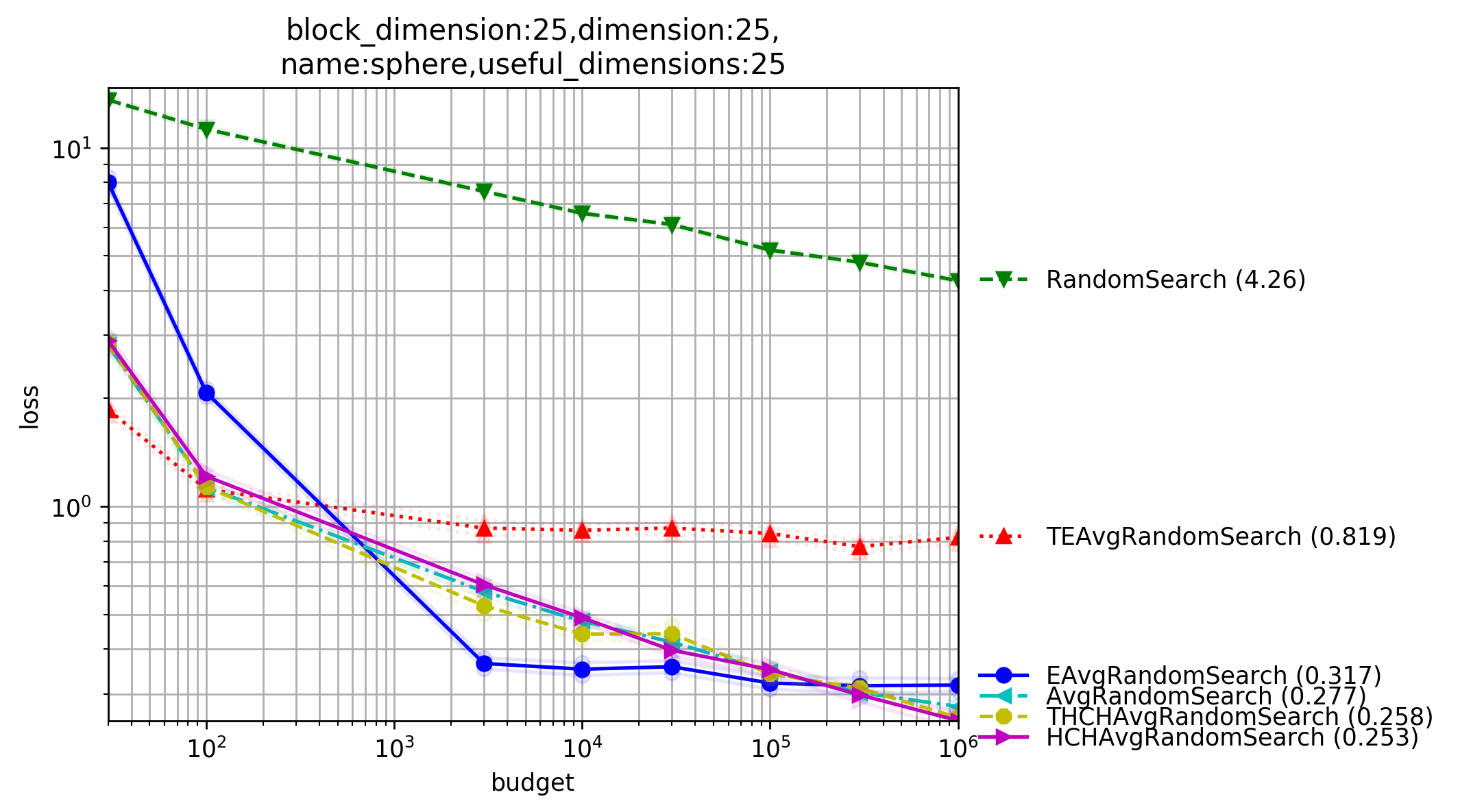}\caption{\label{nevergrad25}Experimental curves comparing various methods for choosing $\mu$ as a function of $\lambda$ in dimension $25$ (Fig. \ref{nevergrad}, continued for dimension 25; see Fig. \ref{nevergrad200} for dimension $200$). Our proposals lead to good results but we notice that they are outperformed by \texttt{TEAvg} and \texttt{EAvg} for Rastrigin: it is better to not take into account non-quasi-convexity because the overall shape is more meaningful that local ruggedness. This phenomenon does not happen for the more rugged HM (Highly Multimodal) function. It also does not happen in dimension 3 or dimension 200 (previous and next figures): in those cases, THCH performed best. Confidence intervals shown in lighter color (they are quite small, and therefore they are difficult to notice).}
\end{figure}
\begin{footnotesize}
\begin{align}
    \mu & = 1 &\mbox{No prefix}\label{eqbeg} \\
    \mu & = \texttt{clip}\left(1, d,\frac{\lambda}{4}\right) &\mbox{Prefix: \texttt{Avg} (averaging)}\label{fabienteytaud}\\
    \mu & = \texttt{clip}\left(1, \infty, \frac{\lambda}{1.1^d}\right) &\mbox{Prefix: \texttt{EAvg} (Exp. Averaging)}\label{eqp1}\\
\mu & = \texttt{clip}\left(1, \min\left(h,\frac{\lambda}{4}\right), d+\frac{\lambda}{1.1^d}\right) &\mbox{Prefix: \texttt{HCHAvg} ($h$ from Convex Hull)}\label{eqmid}\\
    \mu & = \texttt{clip}\left(1, \infty, \frac{\lambda}{1.01^d}\right) &\mbox{Prefix: \texttt{TEAvg} (Tuned Exp. Avg)}\label{eqproved}\\
    \mu & = \texttt{clip}\left(1, \min\left(h,\frac{\lambda}{4}\right), d+\frac{\lambda}{1.01^d}\right) &\mbox{Prefix: \texttt{THCHAvg} (Tuned HCH Avg)}\label{eqlast}
\end{align}\end{footnotesize}
    where $\texttt{clip}(a,b,c)=\max(a,\min(b,c))$ is the projection of $c$ in $[a,b]$ and $h$ is the maximum $i$ such that, for all $j\leq i$, $X_{(j)}$ is on the frontier of the convex hull of $\{X_{(1)},\dots,X_{(j)}\}$ (Sect. \ref{qc}).
Equation \ref{eqbeg} is the naive recommendation ``pick up the best so far''.
Equation \ref{fabienteytaud} existed before the present work: it was, until now, the best rule~\cite{ratio} , overall, in the Nevergrad platform. Equations \ref{eqp1} and \ref{eqproved} are the proposals we deduced from Theorem \ref{thm:asymp}: asymptotically on the sphere, they should have a better rate than Equation \ref{eqbeg}.  Equations \ref{eqmid} and \ref{eqlast} are counterparts of Equations \ref{eqp1} and \ref{eqproved} that combine the latter formulas with ideas from~\cite{ratio}.
 Theorem \ref{thm:asymp} remains true if we add to $\mu$ some constant depending on $d$ so we fine tune our theoretical equation (Eq.~\ref{eqp1}) with the one provided by~\cite{ratio}, so that $\mu$ is close to
 the value in Eq. \ref{fabienteytaud} for moderate values of $\lambda$.
We perform experiments in the open source platform Nevergrad~\cite{nevergrad}.
% \begin{figure}[t]
% \centering
% \includegraphics[trim={0 0 0 0}, clip, width=.48\textwidth]{imgs/hugewidedoenonw/xpresults_block_dimension25,dimension25,namehm,useful_dimensions25.png}
% \includegraphics[trim={0 0 0 20}, clip, width=.48\textwidth]{imgs/hugewidedoenonw/xpresults_block_dimension25,dimension25,namecigar,useful_dimensions25.png}
% \includegraphics[trim={0 0 0 20}, clip, width=.48\textwidth]{imgs/hugewidedoenonw/xpresults_block_dimension25,dimension25,namerastrigin,useful_dimensions25.png}
% \includegraphics[trim={0 0 0 20}, clip, width=.48\textwidth]{imgs/hugewidedoenonw/xpresults_block_dimension25,dimension25,namesphere,useful_dimensions25.png}\caption{\label{nevergrad25}Experimental curves comparing various methods for choosing $\mu$ as a function of $\lambda$ in dimension $25$ (Fig. \ref{nevergrad}, continued for dimension 25; see Fig. \ref{nevergrad200} for dimension $200$). Our proposals lead to good results but we notice that they are outperformed by \texttt{TEAvg} and \texttt{EAvg} for Rastrigin: it is better to not take into account non quasi-convexities because the overall shape is more meaningful that local ruggedness. This phenomenon does not happen for the more rugged HM (Highly Multimodal) function. It also does not happen in dimension 3 or dimension 200 (previous and next figures): in those cases, THCH performed best. Confidence intervals shown in lighter color (they are quite small, and therefore they are difficult to notice).}
% \end{figure}

While previous experiments (Figures \ref{fig:exp_th_c} and \ref{fig:exp_th_nc}) were performed in a controlled ad hoc environment, we work here with more realistic conditions: the sampling is Gaussian (i.e. not uniform in a ball), the objective functions are not all sphere-like, and budgets vary but are not asymptotic.
% \begin{itemize}
%     \item The sampling is Gaussian (i.e. not uniform in a ball), 
%     \item The objective functions are not all sphere-like,
%     \item Budgets vary but are not asymptotic.
% \end{itemize}
Figures \ref{nevergrad}, \ref{nevergrad25}, \ref{nevergrad200} present our results in dimension $3$, $25$ and $200$ respectively.
The objective functions are randomly translated using ${\cal N}(0,0.2I_d)$.
The objective functions are defined as
  $f_{Sphere}(x) = ||x||^2$, 
%f_{griewank}(x) &=& 1+\frac1{4000}f_{sphere}(x)-(\Pi_i cos(x_i/\sqrt{i}))\\
%(s,\mu_1,\mu_2)&=&(1-\frac1{2\sqrt{d+20} + 8.2}, 2.5,-\sqrt{(\mu_1^2-1)/s})\\
%f_{lunacek}(x) &=& \min\left(\sum_i(x_i-\mu_1)^2,d+s\sum_i (x_i-\mu_1)^2\right)+\\
% & & 10\sum_i \left(1-\cos\left(2\Pi (x_i-\mu_1)\right)\right)\\
%\alpha &=& 418.982887\\
%f_{schwefel}(x) &=& \alpha d - \sum_i \sin(\sqrt{|x_i|})\\
%f_{ellipsoid}(x) &=& \sum_i 10^{6(i-2)/(d-1)}x_i^2\\
%f_{discus}(x) &=& f_{sphere}(x) + 999999x_0^2\\
$f_{Cigar}(x) = 10^6 \sum_{i=2}^d x_i^2 + x_1^2$, 
%f_{altcigar}(x) &=& 1000000 \sum_{i=1}^{d-1} x_i^2 + x_d^2\\
$f_{HM}(x)=\sum_{i=1}^d x_i^2\times(1.1+\cos(1/x_i))$,
$f_{Rastrigin}(x) = 10d + f_{sphere}(x) - 10 \sum_i \cos(2\pi x_i)$.
Our proposed equations \texttt{TEAvg} and \texttt{EAvg} are unstable: they sometimes perform excellently (e.g. everything in dimension $25$, Figure \ref{nevergrad25}), but they can also fail dramatically  (e.g. dimension $3$, Figure \ref{nevergrad}).  
Our combinations \texttt{THCHAvg} and \texttt{HCHAvg} perform well: in most settings, \texttt{THCHAvg} performs best. But the gap with the previously proposed \texttt{Avg} is not that big. The use of quasi-convexity as described in Section \ref{qc} was usually beneficial: however, in dimension $25$ for the Rastrigin function, it prevented the averaging from benefiting from the overall ``approximate'' convexity of Rastrigin. This phenomenon did not happen for the ``more'' multimodal function HM, or in other dimensions for the Rastrigin function.
\begin{figure}[t]
\centering
\includegraphics[trim={0 0 0 20}, clip, width=.48\textwidth]{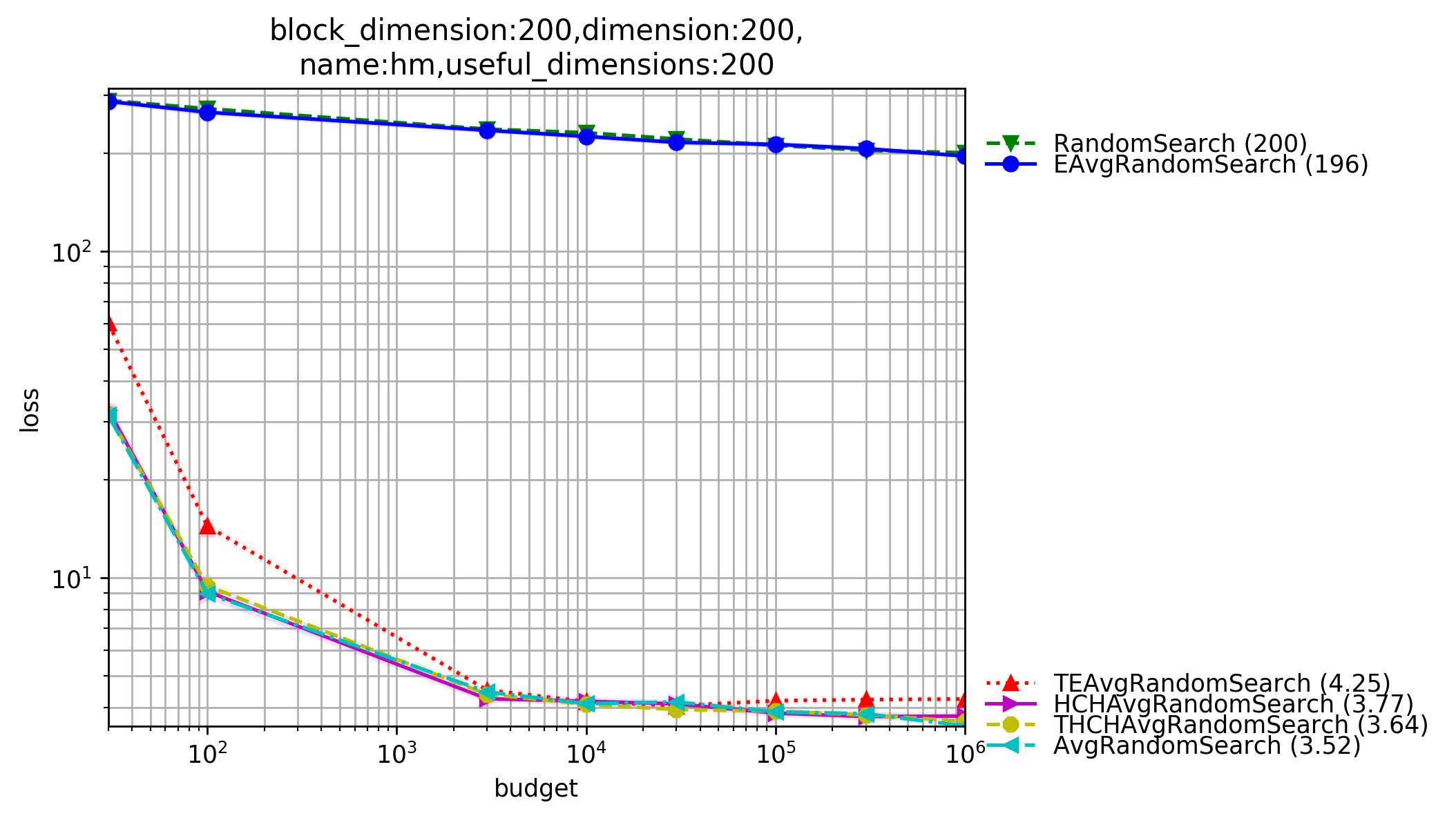}
\includegraphics[trim={0 0 0 20}, clip, width=.48\textwidth]{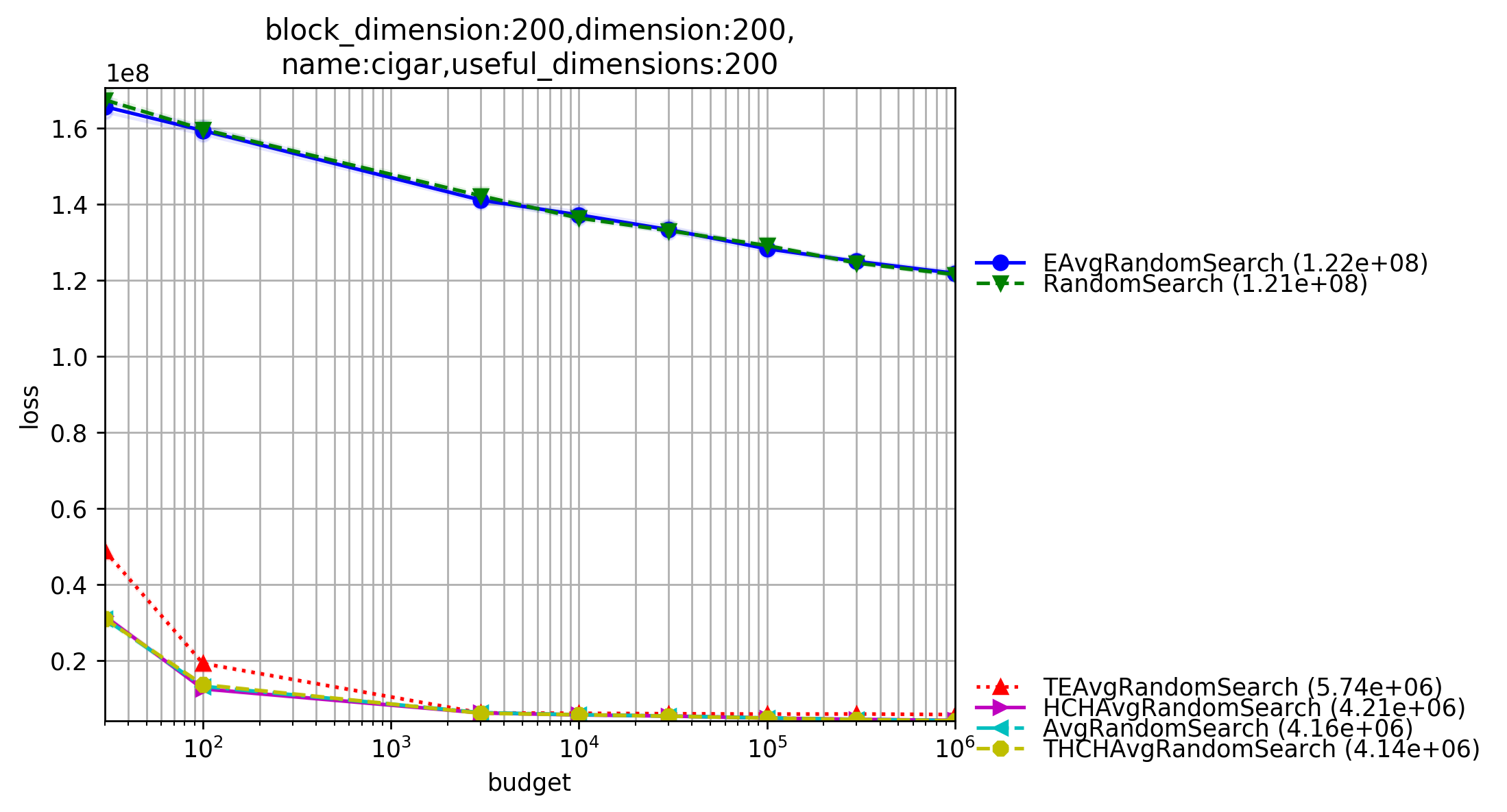}
\includegraphics[trim={0 0 0 20}, clip, width=.48\textwidth]{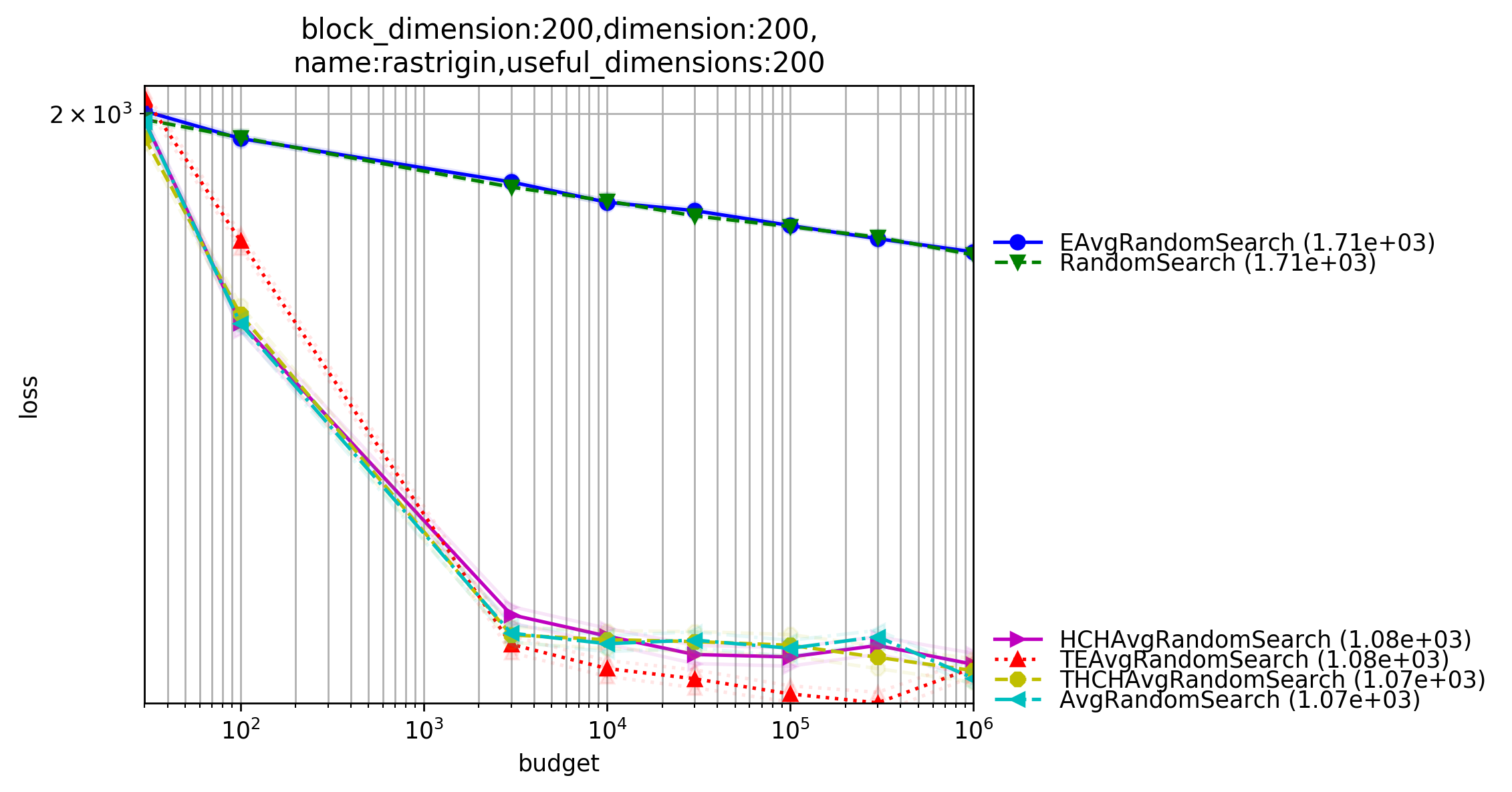}
\includegraphics[trim={0 0 0 20}, clip, width=.48\textwidth]{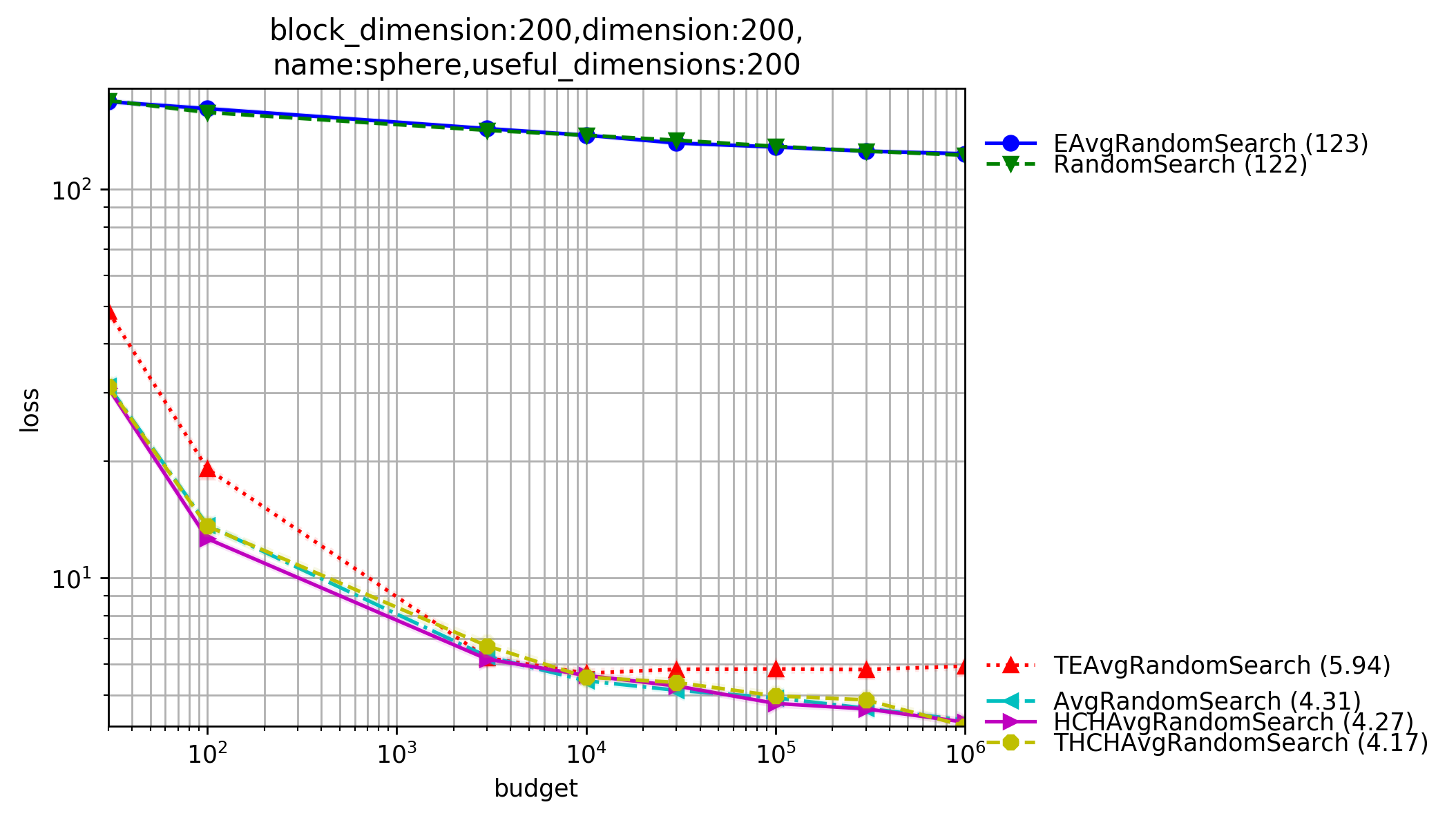}\caption{\label{nevergrad200}Experimental curves comparing various methods for choosing $\mu$ as a function of $\lambda$ in dimension $200$ (Figures \ref{nevergrad} and \ref{nevergrad25}, continued for dimension 200). Confidence intervals shown in lighter color (they are quite small, and therefore they are difficult to notice). Our proposed methods \texttt{THCHAvg} and \texttt{HCHAvg} perform well overall.}
\end{figure}
\section{Conclusion}
We have proved formally that the average of the $\mu$ best is better than the single best in the case of the sphere function (simple regret $O(1/\lambda)$ instead of $O(1/\lambda^{2/d})$) with uniform sampling. We suggested a value $\mu=\lfloor c\lambda\rfloor$ with $0<c<(1-\epsilon)^{d}$. Even better results can be obtained in practice using quasi-convexity, without losing the theoretical guarantees of the convex case on the sphere function. Our results have been successfully implemented in \cite{nevergrad}. The improvement compared to the state of the art, albeit moderate, is obtained without any computational overhead in our method, and supported by a theoretical result.\\
%We nonetheless provide a proof of an improved rate, and did better without any computational overhead in our method. \\

{\bf{Further work.}} Our theorem is limited to a single iteration, i.e. fully parallel optimization, and to the sphere function. Experiments are positive in the convex case, encouraging more theoretical developments in this setting. We did not explore approaches based on surrogate models. Our experimental methods include an automatic choice of $\mu$ in the multimodal case using quasi-convexity, for which the theoretical analysis has yet to be fully developed - we show that this is not detrimental in the convex setting, but not that it performs better in a non-convex setting.
We need an upper bound on the distance between the center of the sampling and the optimum for our results to be applicable (see parameter $\epsilon$): removing this need is an worthy consideration, as such a bound is rarely available in real life.

\newpage
\FloatBarrier
%\crc{It seems contradictory to claim that we can have better results using quasi-convexity but that the theory is missing... OT: ok for you after those}
\bibliographystyle{splncs04}
\bibliography{lsca}%{doe}% yet_another_bib,lsca,games
\end{document}